\documentclass[lettersize,journal]{IEEEtran}
\usepackage{amsmath,amsfonts}
\usepackage{array}
\usepackage{colortbl}
\usepackage{diagbox}
\usepackage{adjustbox}
\usepackage[caption=false,font=normalsize,labelfont=sf,textfont=sf]{subfig}
\usepackage{cite}
\usepackage{textcomp}
\usepackage{stfloats}
\usepackage{amsthm}
\usepackage{url}
\usepackage{bm}
\usepackage{algorithm}
\usepackage{algpseudocode}
\usepackage{verbatim}
\usepackage{graphicx}
\usepackage{multirow} 
\usepackage{color}
\usepackage{diagbox}
\usepackage{siunitx}
 \usepackage{hyperref}
 \usepackage{booktabs}

\def\BibTeX{{\rm B\kern-.05em{\sc i\kern-.025em b}\kern-.08em
    T\kern-.1667em\lower.7ex\hbox{E}\kern-.125emX}}
\usepackage{balance}
\newcommand{\mytheoremname}{\bfseries Theorem}

\newcommand{\mylemmaname}{\bfseries Lemma}
\newcommand{\myassumptionname}{\bfseries Assumption}
\newcommand{\mydefinitionname}{\bfseries Definition}
\newcommand{\mypropositionname}{\bfseries Proposition}
\newcommand{\myclaimname}{\bfseries Claim}

\newtheorem{theorem}{\mytheoremname}

\newtheorem{lemma}[theorem]{\mylemmaname} % 共享编号
\newtheorem{definition}{\mydefinitionname} % 共享编号
\renewcommand{\arraystretch}{1.2}

\begin{document}
\title{From Diffusion to Reaction-Diffusion: A Dynamical-Systems View of Oversmoothing in Hypergraph Neural Networks}
\author{ Zhiheng Zhou, Mengyao Zhou, Yancheng Chen, Dengyi Zhao, Xingqin Qi, Guiying Yan
\thanks{ This work was supported by the National Natural Science Foundation of China (No.12231018 and No.12471330) and the Shandong Provincial Natural
Science Foundation (No.ZR2025MS71). 
%and the Postdoctoral Innovation Program of Shandong Province (No.SDCX-ZG-202603012).
(Corresponding author:  Guiying Yan.)}
\thanks{Zhiheng Zhou, Dengyi Zhao and Xingqin Qi are with the School of Mathematics and Statistics, Shandong University, Weihai, Shandong 264209, China (e-mail: zhouzhiheng@amss.ac.cn; zhaodengyi@mail.sdu.edu.cn; qixingqin@sdu.edu.cn).}
\thanks{Mengyao Zhou, Yancheng Chen and Guiying Yan are with the Academy of Mathematics and Systems Science, Chinese Academy of Sciences and also with the University of Chinese Academy of Sciences, Beijing 100190, China (e-mail: zhoumengyao@amss.ac.cn; chenyancheng22@mails.ucas.ac.cn; yangy@amss.ac.cn).}
% \thanks{Qi Wang is with College of Science, China Agricultural University,
% Beijing, 100083, China (e-mail:  wangqi\_math@cau.edu.cn.)}
% \thanks{Guanghui Wang is with the School of Mathematics, Shandong University, Jinan, Shandong 250100, China (e-mail: 
% ghwang@sdu.edu.cn).}
}

\markboth{Journal of \LaTeX\ Class Files,~Vol.~18, No.~9, September~2020}%
{How to Use the IEEEtran \LaTeX \ Templates}\maketitle

This work has been submitted to the IEEE for possible publication. Copyright may be transferred without notice, after which this version may no longer be accessible.

\begin{abstract}
Higher-order couplings give hypergraph neural networks strong expressive power, but they also make deep propagation vulnerable to rapid representation collapse, especially when multi-way interactions induce strong feature mixing. This paper studies this issue from a dynamical-systems perspective and develops a reaction--diffusion framework for depth-resistant hypergraph learning. Using hypergraph gradient and divergence operators, we formulate message passing as a learnable incidence-level diffusion equation. Analyzing its long-time behavior reveals that pure diffusion generates a global continuous semiflow that exponentially contracts the null-mode-free component of node representations and drives the associated Dirichlet energy to zero. This identifies hypergraph oversmoothing as an intrinsic transverse-energy dissipation process induced by higher-order diffusion. Motivated by this characterization, we propose Hypergraph Neural Reaction--Diffusion (HNRD). Instead of relying on heuristic residual shortcuts, HNRD introduces a reaction term that acts only on the transverse component, combining instantaneous compensation of diffusion-induced dissipation with bounded feedback that stabilizes the transverse energy at a positive learnable level. The resulting dynamics are globally well posed and preserve nontrivial node-discriminative variation, ensuring that the null-mode-free Dirichlet energy remains bounded away from zero. A forward-Euler discretization yields a practical HNRD layer with a step-size condition for stable deep propagation. Experiments on benchmark and synthetic heterophilic hypergraphs show that HNRD consistently outperforms representative hypergraph baselines. Depth-wise, robustness, and runtime analyses further demonstrate that HNRD preserves stable accuracy and nonzero Dirichlet energy under deep propagation and perturbed settings, while maintaining practical efficiency. These results demonstrate that incidence-level dynamical priors offer a principled route to designing deep hypergraph architectures that preserve higher-order expressiveness without collapsing under propagation.
\end{abstract}


\section{Introduction}
Hypergraphs provide a natural mathematical language for modeling multi-way interactions that cannot be faithfully reduced to pairwise edges, such as co-authorship groups, biochemical reactions, co-purchased item sets, and multi-party communication events. Hypergraph neural networks (HGNNs) extend message passing from pairwise neighborhoods to variable-size hyperedges, enabling node representations to aggregate information from higher-order relational structures. Existing models have explored spectral and clique-expansion formulations~\cite{feng2019hypergraph,yadati2019hypergcn}, attention and unified message-passing mechanisms~\cite{zhang2020hyper,huang2021unignn}, and permutation-invariant multiset architectures~\cite{chien2021you,heydari2022message}. These models have achieved strong empirical performance, demonstrating the importance of higher-order relational learning.

Despite this progress, scaling HGNNs to large depth remains challenging. In graph neural networks (GNNs), repeated neighborhood aggregation is known to drive node representations toward indistinguishable configurations, a phenomenon widely referred to as oversmoothing. This issue has been studied from several perspectives, including Laplacian smoothing, asymptotic convergence, Dirichlet-energy decay, Markov-chain interpretations, and representation collapse \cite{li2018deeper,oono2019graph,cai2020note,huang2020tackling,zhao2022comprehensive}. Various graph-level mitigation strategies have been proposed, such as stochastic edge dropping \cite{rong2019dropedge}, normalization \cite{zhao2019pairnorm}, residual and identity mappings \cite{chen2020simple}, decoupled propagation \cite{liu2020towards}, and multi-hop aggregation \cite{xu2018representation}. More recently, reaction--diffusion graph neural models have introduced reaction mechanisms to balance diffusion-induced smoothing and feature sharpening \cite{choi2023gread,eliasof2024graph}. However, these analyses and architectures are primarily developed for pairwise graphs, where propagation is governed by edge-based Laplacian operators; oversmoothing in hypergraphs, however, acquires a distinct structural signature.

Unlike pairwise graphs, a single hyperedge can simultaneously couple multiple nodes, producing stronger and more global mixing. Consequently, repeated hypergraph propagation may suppress node-wise heterogeneity more rapidly than ordinary graph propagation. Recent studies have begun to explicitly address this issue, including Deep-HGCN, which analyzes and mitigates oversmoothing in deep HGNNs \cite{chen2022preventing}, framelet-based hypergraph networks that preserve high-frequency components \cite{li2025deep}, sheaf-based hypergraph models that study hypergraph Dirichlet-energy dissipation \cite{duta2023sheaf}, and implicit hypergraph networks that analyze stability from a fixed-point perspective \cite{li2025implicit}. However, these discrete architectures often rely on heuristic residual connections, spectral truncations, or fixed-point formulations, and lack a principled mechanism to dynamically balance feature smoothing and diversity preservation as depth increases. A general dynamical explanation of why hypergraph diffusion causes collapse, and how to design hypergraph neural dynamics that provably avoid it, remains largely uncharted in the context of continuous neural dynamics.

To address this question, we develop a dynamical-systems framework for understanding and mitigating oversmoothing in HGNNs. Using a hypergraph gradient--divergence pair, we formulate message passing as a learnable incidence-level diffusion equation. This formulation makes it possible to analyze deep hypergraph propagation through semiflow theory and energy methods. Our analysis shows that pure hypergraph neural diffusion exponentially contracts the null-mode-free component of node representations and drives the associated Dirichlet energy to zero. This reveals a key mechanism behind hypergraph oversmoothing: higher-order diffusion progressively dissipates transverse, node-discriminative variation.

Guided by this insight, we introduce Hypergraph Neural Reaction--Diffusion (HNRD), an incidence-level reaction--diffusion model designed to counteract transverse-energy collapse. The reaction term acts only on the null-mode-free component and combines two effects: instantaneous compensation of diffusion-induced dissipation and bounded feedback toward a positive learnable energy level. This design preserves the stabilizing role of hypergraph diffusion while preventing node representations from collapsing into the null-mode subspace. We establish global well-posedness, positive preservation of the null-mode-free Dirichlet energy, and a stable forward-Euler discretization with a step-size condition for deep propagation.

Extensive experiments are conducted on academic, real-world, and synthetic heterophilic hypergraph datasets. HNRD achieves the best overall performance among representative graph diffusion models, standard HGNNs, expressive HGNNs, and deep/stable HGNNs. Depth-wise analysis further shows that HNRD maintains stable accuracy and nonzero Dirichlet energy under very deep propagation, providing empirical support for the proposed non-collapse theory. Additional robustness, parameter, visualization, and runtime analyses demonstrate that HNRD is effective, stable, and computationally practical under diverse settings.

The main contributions of this paper are summarized here:
\begin{itemize}
\item We identify oversmoothing in HGNNs as an intrinsic transverse-energy dissipation process induced by higher-order diffusion, and prove that pure hypergraph neural diffusion exponentially contracts the null-mode-free component and drives the Dirichlet energy to zero.

\item We introduce HNRD, an incidence-level hypergraph reaction--diffusion model whose reaction term explicitly compensates transverse diffusion dissipation and stabilizes a nonzero transverse-energy level.

\item We establish global well-posedness, transverse-energy preservation, and positive Dirichlet-energy lower bounds for HNRD, and further derive a discrete HNRD layer with a provable step-size condition for stable deep propagation.

\item We validate HNRD through extensive experiments on benchmark, real-world, and synthetic heterophilic hypergraphs, demonstrating its effectiveness, deep stability, robustness, and practical efficiency.
\end{itemize}

\section{Related Work}
\subsection{Hypergraph Neural Networks}

Hypergraph neural networks extend graph representation learning from pairwise edges to higher-order relational structures. Early representative models include spectral HGNNs based on normalized hypergraph Laplacians \cite{feng2019hypergraph} and graph-approximation methods such as HyperGCN \cite{yadati2019hypergcn}, which convert each hyperedge into representative pairwise connections. Subsequent architectures further improve the flexibility of hypergraph message passing. Hyper-SAGNN \cite{zhang2020hyper} introduces self-attention for variable-size hyperedges, HNHN \cite{dong2020hnhn} explicitly updates both node and hyperedge representations, UniGNN \cite{huang2021unignn} unifies graph and hypergraph propagation, and AllSet \cite{chien2021you} formulates hypergraph aggregation through learnable multiset functions. More general message-passing formulations have also been developed to characterize node--hyperedge interactions under a unified paradigm \cite{heydari2022message}.

Recent studies have pursued two complementary routes to deepen the theoretical foundations of hypergraph learning. The first route is grounded in energy principles and continuous dynamics. Energy-based formulations interpret hypergraph propagation through learnable regularized energy minimization \cite{wang2023hypergraph}; Hypergraph Dynamic System models representation evolution using ODE-based control and diffusion mechanisms \cite{yan2024hypergraph}; HNDiffN introduces coupled continuous-time diffusion equations on nodes and hyperedges \cite{lu2025hypergraph}; and PDE-inspired frameworks characterize hypergraph message passing as nonlinear anisotropic diffusion governed by gradient--divergence operators and learnable incidence-level coefficients \cite{zhou2026hypergraph}. The second route enhances spectral expressiveness and stability through alternative architectures, including framelet-based high-frequency preservation \cite{li2025deep}, sheaf-based diffusion on hypergraphs \cite{choi2025hypergraph}, and implicit fixed-point propagation with stability guarantees \cite{li2025implicit}.

Despite their diversity, these approaches share a common limitation: their analyses are either confined to the discrete layer level or tied to specific propagation rules. They cannot fully characterize the long-time asymptotic behavior of learnable hypergraph diffusion, nor do they provide a unified dynamical explanation of why repeated propagation suppresses node-discriminative information. A principled continuous-time framework that explains the mechanism of collapse and provides provable mitigation remains absent. In contrast, our work provides a continuous-time dynamical characterization of oversmoothing in hypergraph neural networks. We prove that pure hypergraph diffusion inevitably contracts all transverse components toward a null-mode subspace, giving a precise mathematical explanation of representation collapse. To counteract this intrinsic limitation, we further introduce a reaction--diffusion mechanism that compensates transverse dissipation and provably preserves node-discriminative variation, distinguishing our approach from heuristic residual connections, spectral truncations, and fixed-point stabilization schemes.

\subsection{Oversmoothing in Graph and Hypergraph Neural Networks}
Oversmoothing is a fundamental obstacle to building deep graph neural networks. Early studies interpret this phenomenon from the perspective of Laplacian smoothing, showing that repeated graph convolution tends to make node representations increasingly similar \cite{li2018deeper}. Subsequent analyses further characterize oversmoothing through asymptotic convergence, expressive-power decay, Dirichlet-energy dissipation, Markov-chain mixing, and representation collapse \cite{oono2019graph,cai2020note,huang2020tackling,zhao2022comprehensive}. To alleviate this issue, various architectural strategies have been proposed, including stochastic edge dropping \cite{rong2019dropedge}, normalization-based feature separation \cite{zhao2019pairnorm}, residual and identity mappings \cite{chen2020simple}, decoupled propagation \cite{liu2020towards}, and layer-wise feature aggregation \cite{xu2018representation}. More recent studies revisit oversmoothing from diffusion and representation-rank perspectives, such as operator semigroup analysis for diffusion-based GNNs \cite{zhao2025understanding} and rank-based characterizations of representation collapse \cite{deidda2025rethinking}. These works provide important insights into graph oversmoothing, but they mainly focus on pairwise graph operators.

Continuous-depth and reaction--diffusion models offer another route for understanding and mitigating oversmoothing. GRAND++ first introduces a source term into graph neural diffusion to prevent constant-state convergence \cite{thorpe2022grand++}. GREAD subsequently formulates graph learning as a reaction--diffusion process, where the reaction term balances diffusion-induced smoothing and feature sharpening \cite{choi2023gread}. RDGNN further develops reaction--diffusion models inspired by Turing instabilities and studies their continuous and discretized dynamics \cite{eliasof2024graph}. These methods demonstrate the usefulness of PDE-inspired mechanisms for controlling graph diffusion. However, their formulations are built on pairwise graph Laplacians or graph diffusion operators, and therefore do not directly capture the incidence-level higher-order mixing induced by hyperedges.

Oversmoothing in hypergraph neural networks has a distinct structural character because one hyperedge can simultaneously couple multiple nodes, leading to stronger and more global feature mixing. Deep-HGCN explicitly studies over-smoothing in HGNNs and proposes a deep hypergraph convolutional architecture to maintain representation heterogeneity \cite{chen2022preventing}. FrameHGNN analyzes oversmoothing in deep HGNNs from a spectral perspective and introduces tight framelet transforms with low- and high-pass components to preserve multi-frequency information \cite{li2025deep}. Sheaf-based hypergraph networks use hypergraph sheaf Laplacians to enrich higher-order diffusion and study Dirichlet-energy behavior \cite{duta2023sheaf}, while implicit hypergraph neural networks formulate propagation as a nonlinear fixed-point problem with stability and oversmoothing analysis \cite{li2025implicit}. More recently, Ricci-flow-guided hypergraph neural diffusion uses geometric flow to regulate hypergraph diffusion and alleviate feature homogenization \cite{zhou2026tackling}. Despite these advances, existing approaches are either tied to discrete architectural modifications, spectral/high-frequency preservation, sheaf-based propagation, geometric rewiring, or fixed-point stabilization. They do not provide a unified continuous-time explanation of why learnable hypergraph diffusion contracts node-discriminative components, nor do they offer a reaction mechanism that explicitly compensates transverse diffusion dissipation. In contrast, our HNRD characterizes oversmoothing as null-mode attraction with Dirichlet-energy decay and introduces an incidence-level reaction--diffusion dynamics that provably stabilizes nontrivial transverse variation.

\section{Preliminaries}
Let \(\mathcal{G}=(\mathcal{V},\mathcal{E},\omega)\) denote a weighted hypergraph, where \(\mathcal{V}\) is the set of nodes with \(|\mathcal{V}|=n\), \(\mathcal{E}\) is the set of hyperedges, and \(\omega:\mathcal{E}\rightarrow \mathbb{R}_{>0}\) assigns a positive weight \(\omega_e\) to each hyperedge \(e\in\mathcal{E}\). Each hyperedge contains a subset of nodes, namely \(e\subseteq \mathcal{V}\), with cardinality \(|e|\geq 2\). For a node \(v\in\mathcal{V}\), its degree is given by \(d_v=\sum_{e\in\mathcal{E}}\mathbb{I}(v\in e),\)
where \(\mathbb{I}(\cdot)\) is the indicator function. We denote the node--hyperedge incidence set by
\(\mathcal{I}=\{(e,v)\mid e\in\mathcal{E}, v\in e\},\)
with \(|\mathcal{I}|=N=\sum_{e\in\mathcal{E}}|e|.\)

The node degree matrix \(\mathbf{D_v}\in\mathbb{R}^{n\times n}\) is defined as a diagonal matrix with \(\mathbf{D_v}(v,v)=d_v.\)
We further define the incidence-level weight matrix
\(\mathbf{\Omega}_{\mathcal I}=\operatorname{diag}(\omega_e)_{(e,v)\in\mathcal{I}}
\in\mathbb{R}^{N\times N},\)
where the weight of each incidence pair \((e,v)\) is inherited from its corresponding hyperedge \(e\). The initial node feature matrix is denoted by \(\mathbf{X}_{\text{0}}\), and \(\|\cdot\|_F\) denotes the Frobenius norm. The node signal space is defined as
\(L(\mathcal{V}):=\{\mathbf f:\mathcal{V}\to\mathbb{R}^d\}
\cong\mathbb{R}^{n\times d},\) while the incidence-level signal space is defined as \(L(\mathcal{E},\mathcal{V})
:=\{\mathbf g:\mathcal{I}\to\mathbb{R}^d\}\cong\mathbb{R}^{N\times d}.\)
Here, \(d\) denotes the embedding dimension.

\subsection{Hypergraph Neural Diffusion}
Following the hypergraph neural differential formulation in~\cite{zhou2026hypergraph}, we characterize diffusion on a weighted hypergraph \(\mathcal{G}=(\mathcal{V},\mathcal{E},\omega)\) using incidence-level gradient and divergence operators. The gradient lifts node signals to node--hyperedge variations, while the divergence aggregates the resulting incidence-level flows back to nodes.

For a node function \(f\in L(V)\), the hypergraph gradient on each incidence pair \((e,v)\in I\) is defined as
\begin{equation}
(\nabla \mathbf f)(e,v)
=\frac{\mathbf f(v)}{\sqrt{d_v}}
-\frac{1}{|e|}
\sum_{u\in e}
\frac{\mathbf f(u)}{\sqrt{d_u}} .
\label{eq:hypergraph_gradient}
\end{equation}

The associated divergence operator maps incidence-level flows back to the node domain. For \(\mathbf g\in L(\mathcal E,\mathcal V)\), it is defined by
\begin{equation}
(\operatorname{div}\mathbf g)(v)
=
\sum_{e\ni v}
\frac{\omega_e}{\sqrt{d_v}}
\left(
\mathbf g(e,v)
-
\frac{1}{|e|}
\sum_{u\in e}
\mathbf g(e,u)
\right).
\label{eq:divergence}
\end{equation}

Using the above operators, the hypergraph neural diffusion equation is formulated as
\begin{equation}
\frac{\partial \mathbf X(t)}{\partial t}
=
-\operatorname{div}
\left[
\mathbf A_\theta(\mathbf X(t))\nabla \mathbf X(t)
\right],
\label{eq:deterministic_hypergraph_diffusion}
\end{equation}
where \(\mathbf X(t)\in L(\mathcal V)\) denotes the time-dependent node representation. The learnable operator \(\mathbf A_\theta(\mathbf X(t))\in\mathbb R^{N\times N}\) modulates the diffusion strength on the incidence space. Specifically, it is taken as a diagonal operator
\begin{equation}
\mathbf A_\theta(\mathbf X(t))
=
\operatorname{diag}
\left(
a_\theta(\mathbf X_v(t),\mathbf X_e(t))
\right)_{(e,v)\in\mathcal I},
\label{eq:modulation_operator}
\end{equation}
where \(\mathbf X_v(t)\) denotes the representation of node \(v\), and \(\mathbf X_e(t)\) denotes an aggregated representation of hyperedge \(e\). Thus, each diagonal entry \(a_\theta(\mathbf X_v(t),\mathbf X_e(t))\) adaptively determines the diffusion intensity between node \(v\) and hyperedge \(e\).

In matrix form~\cite{zhou2026hypergraph}, Eq.~\eqref{eq:deterministic_hypergraph_diffusion} becomes
\begin{equation}
\frac{\partial \mathbf X(t)}{\partial t}
=
-\mathbf G^\top
\mathbf A_\theta(\mathbf X(t))
\mathbf G\mathbf X(t).
\label{eq:G_matrix_hypergraph_diffusion}
\end{equation}
In this form, \(\mathbf G\in\mathbb R^{N\times n}\) extracts incidence-level variations from node representations, \(\mathbf A_\theta(\mathbf X(t))\) adaptively weights these variations, and \(\mathbf G^\top\in\mathbb R^{n\times N}\) aggregates the resulting flows back to the node domain.

\subsection{Semiflow and Attracting Sets}

We briefly introduce several dynamical-systems notions used in the subsequent analysis~\cite{hale2010asymptotic,temam2012infinite}. Throughout this paper, the phase space is the finite-dimensional Banach space
\(\mathcal X=\mathbb{R}^{n\times d},\)
equipped with the Frobenius norm.

\begin{definition}
\(\{\mathbf S(t)\}_{t\ge0}\), with \(\mathbf S(t):\mathcal X\to \mathcal X\), is called a continuous semiflow on (X) if
\begin{equation}
\mathbf S(0)=\mathbf I,\quad
\mathbf S(t+s)=\mathbf S(t)\circ \mathbf S(s),\quad t,s\ge 0,
\end{equation}
and the map \((t,x)\mapsto \mathbf S(t)x\) is continuous from \([0,\infty)\times \mathcal X\) to \(\mathcal X\). 
\end{definition}

\begin{definition}
A set \(\mathcal A\subset \mathcal X\) is called an attracting invariant set if for all $t\ge 0$, we have
\(\mathbf S(t)\mathcal A=\mathcal A,\) and for every bounded set \(B\subset\mathcal X\), when $t\to\infty$, then
\begin{equation}
\operatorname{dist}_{\mathcal X}(\mathbf S(t)B,\mathcal A)
:=\sup_{x\in B}\inf_{y\in\mathcal A}
\|\mathbf S(t)x-y\|_F\longrightarrow 0.
\end{equation}
\end{definition}
Attracting invariant sets are used to describe the asymptotic behavior of hypergraph neural diffusion and to interpret oversmoothing as convergence toward a low-dimensional null-mode subspace.

\section{Dynamic Systems View of Oversmoothing}
This section analyzes the long-time behavior of the hypergraph neural diffusion dynamics introduced above. The goal is to show that pure hypergraph diffusion naturally drives node representations toward a low-dimensional null-mode subspace, thereby providing a dynamical-systems interpretation of oversmoothing.

We study the dynamical properties of the hypergraph neural diffusion equation
\begin{equation}
\frac{\partial \mathbf X(t)}{\partial t}
=
-\mathbf G^{\top}\mathbf A_\theta(\mathbf X(t))\mathbf G\mathbf X(t),
\qquad
\mathbf X(0)=\mathbf X_0\in\mathcal X .
\label{eq:diffusion_semiflow_equation}
\end{equation}
where $\mathbf{A}_{\theta}(\mathbf{X}(t))=\operatorname{diag}\big(a_{\theta}(\mathbf{X}_v(t),\mathbf{X}_e(t))\big)_{(e,v)\in\mathcal{I}}\in\mathbb{R}^{N\times N}$ is a learnable diagonal modulation matrix defined over the node--hyperedge incidence set $\mathcal I$.

For each incidence pair $(e,v)\in\mathcal I$, we compute a hyperedge-level representation $\mathbf{X}_e(t)=\operatorname{Agg}_u(\{\mathbf{X}_u(t) : u \in e\})$, where $\operatorname{Agg}$ is a permutation-invariant aggregator, such as mean or max pooling. The compatibility between node \(v\) and hyperedge \(e\) is then modeled by
\begin{equation}
\begin{aligned}
r^{\theta}_{(e,v)} &= \sigma\left(\operatorname{MLP}_{\theta}(\mathbf{X}_v(t) \| \mathbf{X}_e(t))\right),\\
a_\theta(\mathbf{X}_v(t), \mathbf{X}_e(t))&= (1-\varepsilon)\frac{\exp\big(r^{\theta}_{(e,v)}\big)}{\sum_{e' \ni v} \exp\big( r^{\theta}_{(e',v)} \big)}+\frac{\varepsilon}{d_v},
\end{aligned}
\end{equation}
where $\|$ denotes concatenation, $\sigma$ is a $\operatorname{LeakyReLU}$ activation function, and \(0<\varepsilon<1\) is a small regularization constant that prevents the normalized incidence weights from degenerating to zero. This formulation guarantees that $a_\theta(\mathbf{X}_v(t), \mathbf{X}_e(t)) > 0$ and \( \sum_{e \ni v} a_\theta(\mathbf{X}_v(t), \mathbf{X}_e(t)) = 1\). We first record the basic properties induced by this normalized modulation.
\begin{lemma}
\label{lem:normalized_diagonal_modulation}
Assume that the hypergraph \(\mathcal G=(\mathcal V,\mathcal E,\omega)\) is connected. Let \(
d_{\max}:=\max_{v\in\mathcal V}|d_v|.\) Then the following properties hold:
\begin{enumerate}
\item[\textnormal{(i)}] \(\mathbf A_\theta(\mathbf X)\) is locally Lipschitz with respect to \(\mathbf X\);
\item[\textnormal{(ii)}] \(\mathbf A_\theta(\mathbf X)\) is uniformly positive definite. More precisely,
\begin{equation}
\mathbf I_N\succeq\mathbf A_\theta(\mathbf X)
\succeq
a_{\min}\mathbf I_N,
\quad
a_{\min}:=\frac{\varepsilon}{d_{\max}}>0;
\end{equation}
\item[\textnormal{(iii)}] the hypergraph gradient matrix satisfies
\begin{equation}
\ker(\mathbf G)=\operatorname{span}\{\phi\},
\quad
\phi=\mathbf D_v^{1/2}\mathbf 1 .
\end{equation}
Equivalently, for matrix-valued node representations,
\begin{equation}
\ker(\mathbf G)=
\mathcal S_\phi
:=
\{\phi c:\ c\in\mathbb R^{1\times d}\}.
\end{equation}
\end{enumerate}
\end{lemma}
The proof of Lemma~\ref{lem:normalized_diagonal_modulation} is provided in Appendix~\ref{app:Lemma1}. 
The lemma guarantees the regularity and coercivity of the diffusion vector field and identifies the lifted null-mode subspace \(\mathcal S_\phi=\ker(\mathbf G)\), thereby providing the structural foundation for the semiflow analysis below.

\begin{theorem}
\label{thm:diffusion_semiflow}
Under the conditions of Lemma~\ref{lem:normalized_diagonal_modulation}, for every initial condition \(\mathbf X_0\in\mathcal X\), Eq.~\eqref{eq:diffusion_semiflow_equation} admits a unique global classical solution. Consequently, the solution operator
\begin{equation}
\mathbf S_D(t)\mathbf X_0:=\mathbf X(t;\mathbf X_0),
\quad t\ge 0,
\end{equation}
defines a global continuous semiflow \(\{\mathbf S_D(t)\}_{t\ge0}\) on \(\mathcal X\).
\end{theorem}
The proof of this result has been established in previous work~\cite{zhou2026hypergraph}; here we only use the induced semiflow formulation. It provides the basic well-posedness foundation for analyzing the long-time behavior of hypergraph neural diffusion.

We next examine whether the diffusion semiflow drives the representations toward a collapsed state. Let \(\mathbf \Pi_\phi\) be the orthogonal projector onto \(\mathcal S_\phi\), and define \(\mathbf Q_\phi=\mathbf I-\mathbf \Pi_\phi.\) Then \(\mathbf Q_\phi\mathbf X\) represents the transverse component of \(\mathbf X\), i.e., the part orthogonal to the null mode. Oversmoothing is characterized by the decay of this component along the diffusion semiflow.

\begin{theorem}
\label{thm:null_mode_attraction}
Under the conditions of Lemma~\ref{lem:normalized_diagonal_modulation}, the null-mode subspace \(\mathcal S_\phi\)
is invariant under the diffusion semiflow \(\{\mathbf S_D(t)\}_{t\ge0}\). Moreover, let \(
\lambda_2
:=
\min_{\mathbf Z\in \mathcal S_\phi^\perp,\ \|\mathbf Z\|_F=1}
\|\mathbf G\mathbf Z\|_F^2.
\) Then \(\lambda_2>0\), and for every \(\mathbf X_0\in\mathcal X\), and all $t\ge0$
\begin{equation}
\operatorname{dist}(\mathbf S_D(t)\mathbf X_0,\mathcal S_\phi)
=
\|\mathbf Q_\phi \mathbf S_D(t)\mathbf X_0\|_F
\le
e^{-\gamma_D t}
\|\mathbf Q_\phi\mathbf X_0\|_F,
\end{equation}
where \(\gamma_D:=a_{\min}\lambda_2>0 .\)
\end{theorem}

The proof of Theorem~\ref{thm:null_mode_attraction} is provided in Appendix~\ref{app:Theorem3}. Theorem~\ref{thm:null_mode_attraction} establishes that hypergraph neural diffusion exponentially suppresses all components orthogonal to the null-mode subspace. Hence, the diffusion trajectory is attracted to the low-dimensional subspace \(\mathcal S_\phi\), which gives a dynamical-systems characterization of oversmoothing.

To obtain an operator-based measure of this collapse, we further introduce a null-mode-free Dirichlet energy. Let \(\mathbf L_H\in\mathbb R^{n\times n}\) be a symmetric positive semidefinite hypergraph Laplacian satisfying
\(
\mathbf L_H\phi=0,\) and \(\ker(\mathbf L_H)=\operatorname{span}\{\phi\}.
\)
Define the null-mode-free Dirichlet energy by
\begin{equation}
\mathcal E_\phi(\mathbf X)
:=
\frac12
\left\langle
\mathbf Q_\phi\mathbf X,
\mathbf L_H \mathbf Q_\phi\mathbf X
\right\rangle_F .
\end{equation}
This energy vanishes exactly on the null-mode subspace and therefore provides a quantitative measure of the remaining node-wise variation.
\begin{theorem}
\label{thm:diffusion_dirichlet_decay}
Under the conditions of Lemma~\ref{lem:normalized_diagonal_modulation}, the Dirichlet energy of the diffusion semiflow converges to zero. More precisely, for every \(\mathbf X_0\in\mathcal X\),
\begin{equation}
\begin{aligned}
\mathcal E_\phi(\mathbf S_D(t)\mathbf X_0)
\le
\frac12
\lambda_{\max}(\mathbf L_H)
e^{-2\gamma_D t}
\|\mathbf Q_\phi\mathbf X_0\|_F^2,
\quad t\ge0,\\
\mathcal E_\phi(\mathbf S_D(t)\mathbf X_0)\to0
\Longleftrightarrow
\operatorname{dist}(\mathbf S_D(t)\mathbf X_0,\mathcal S_\phi)\to0,\;as\;t\to\infty.
\end{aligned}
\end{equation}
\end{theorem}
The proof of Theorem~\ref{thm:diffusion_dirichlet_decay} is provided in Appendix~\ref{app:Theorem4}. Together, Theorems~\ref{thm:null_mode_attraction} and \ref{thm:diffusion_dirichlet_decay} provide a complete dynamical characterization of hypergraph neural diffusion. The diffusion semiflow is globally well-defined, preserves the null-mode subspace, exponentially contracts all transverse components, and drives the associated Dirichlet energy to zero. Hence, diffusion inevitably removes null-mode-free information in the long-time regime, which explains why deeply stacked diffusion-based hypergraph propagation may suffer from oversmoothing.

This limitation motivates the introduction of a reaction mechanism. The reaction term is designed to counterbalance the transverse dissipation induced by the diffusion operator, so that the null-mode-free component remains non-degenerate while the stabilizing effect of hypergraph diffusion is retained.

\section{Hypergraph Neural Reaction–-Diffusion}
\label{sec:hnrd}

\begin{figure}[ht]
\centering
{\includegraphics[scale=0.77]{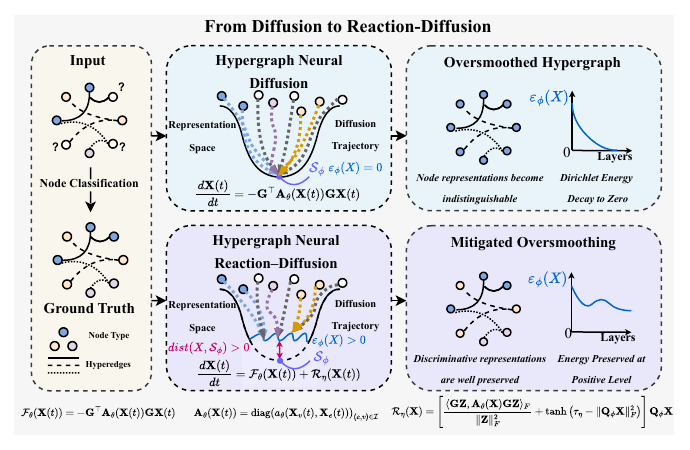}}
\caption{The upper panel shows that hypergraph diffusion drives node representations toward indistinguishability and causes the Dirichlet energy to decay, resulting in oversmoothing. The lower panel introduces a learnable reaction term to preserve representation diversity and maintain positive energy, thereby mitigating oversmoothing.}
\label{HNRD}
\end{figure}

The previous section shows that hypergraph neural diffusion intrinsically drives node representations toward the null-mode subspace \(\mathcal S_\phi\). Although such diffusion provides a useful smoothing and denoising mechanism, its long-time behavior inevitably suppresses the transverse component \(Q_\phi\mathbf X(t)\), leading to oversmoothing. To prevent this collapse, we introduce a learnable reaction mechanism that counterbalances the transverse dissipation induced by diffusion while retaining the stabilizing effect of the hypergraph diffusion operator.

Recall the learnable hypergraph diffusion equation
\begin{equation}
\frac{d\mathbf X(t)}{dt}
=\mathcal F_\theta(\mathbf X(t))=
-\mathbf G^\top
\mathbf A_\theta(\mathbf X(t))
\mathbf G\mathbf X(t),
\label{eq:pure_learnable_diffusion}
\end{equation}
where \(\mathbf A_\theta(\mathbf X(t))\) is a learnable diagonal modulation matrix. Let
\(
\mathbf Z(t):=\mathbf Q_\phi\mathbf X(t),\) and \(
s(t):=\|\mathbf Z(t)\|_F^2 .\)
The quantity \(s(t)\) measures the null-mode-free, or transverse, energy of the node representation. Under pure diffusion, \(s(t)\) decays to zero, which corresponds to attraction toward \(\mathcal S_\phi\).

To counteract this collapse, we design a reaction term acting only on the transverse component. The reaction should satisfy two requirements. First, it should compensate the instantaneous transverse dissipation produced by the diffusion operator. Second, it should introduce a bounded saturating feedback that drives the transverse energy toward a nonzero learnable level. These two requirements lead to the following reaction design. We define the instantaneous diffusion dissipation rate
\begin{equation}
R_\theta(\mathbf X)
:=
\frac{
\left\langle
\mathbf G\mathbf Z,
\mathbf A_\theta(\mathbf X)\mathbf G\mathbf Z
\right\rangle_F
}{
\|\mathbf Z\|_F^2
}.
\label{eq:instantaneous_diffusion_dissipation}
\end{equation}
When $\mathbf Z=\mathbf Q_\phi\mathbf X=0$, we define $R_\theta(\mathbf X)=0$. The scalar \(R_\theta(\mathbf X)\) measures the diffusion-induced dissipation per unit transverse energy. It depends on both the hypergraph structure and the learnable diffusion modulation \(\mathbf A_\theta(\mathbf X)\).

We then define the learnable reaction term as
\begin{equation}
\mathcal R_\eta(\mathbf X)
=
\left[
R_\theta(\mathbf X)
+
\tanh
\left(
\tau_\eta-\|\mathbf Q_\phi\mathbf X\|_F^2
\right)
\right]
\mathbf Q_\phi\mathbf X,
\label{eq:learnable_reaction_term}
\end{equation}
where \(\tau_\eta=\operatorname{softplus}(\eta)>0\) is a learnable transverse-energy level. The first term \(R_\theta(\mathbf X)\mathbf Q_\phi\mathbf X\) compensates the instantaneous transverse dissipation of the diffusion operator, while the bounded feedback term \(\tanh\left(
\tau_\eta-\|\mathbf Q_\phi\mathbf X\|_F^2\right)\mathbf Q_\phi\mathbf X\) drives the transverse energy toward the learnable nonzero level \(\tau_\eta\). Since the feedback coefficient is bounded between \(-1\) and \(1\), the reaction mechanism avoids unbounded reaction strength.

The Hypergraph Neural Reaction--Diffusion (HNRD) is
\begin{equation}
\frac{d\mathbf X(t)}{dt}
=\mathcal F_\theta(\mathbf X(t))
+\mathcal R_\eta(\mathbf X(t)),\quad
\mathbf X(0)=\mathbf X_0\in\mathcal X .
\label{eq:hnrd_equation}
\end{equation}

\subsection{Theoretical Analysis}
The reaction term in Eq.~\eqref{eq:learnable_reaction_term} is designed to preserve the null-mode decomposition. Since it is proportional to \(\mathbf Q_\phi\mathbf X\), it acts only on the transverse component and does not directly perturb the null-mode component \(\mathbf \Pi_\phi\mathbf X\). Therefore, the reaction term does not shift the null-mode subspace itself; instead, it regulates whether the trajectory collapses toward this subspace. The following properties summarize the key structural role of the reaction term.
\begin{lemma}
\label{lem:bounded_reaction_coefficient}
Under the conditions of Lemma~\ref{lem:normalized_diagonal_modulation}, for every \(\mathbf X\notin\mathcal S_\phi\),
\(0\le R_\theta(\mathbf X)\le\lambda_{\max}(\mathbf L_{H}).\)
Consequently, the scalar reaction coefficient
\begin{equation}
R_\theta(\mathbf X)
+
\tanh
\left(
\tau_\eta-\|\mathbf Q_\phi\mathbf X\|_F^2
\right)
\end{equation}
is uniformly bounded.
\end{lemma}

The proof of Lemma~\ref{lem:bounded_reaction_coefficient} is provided in Appendix~\ref{app:Lemma5}. This lemma shows that the proposed reaction term does not introduce an unbounded reaction coefficient. The boundedness of the reaction coefficient provides the key regularity needed for the reaction--diffusion dynamics. We next show that the proposed HNRD equation is well posed and generates a global continuous semiflow on the finite-dimensional phase space \(\mathcal X\).

\begin{theorem}
\label{thm:hnrd_wellposedness}
Under the conditions of Lemma~\ref{lem:normalized_diagonal_modulation}, for every initial condition \(\mathbf X_0\in\mathcal X\), the HNRD equation Eq.~\eqref{eq:hnrd_equation} admits a unique global classical solution. Consequently, the solution operator
\begin{equation}
S_R(t)\mathbf X_0:=\mathbf X(t;\mathbf X_0),
\quad t\ge0,
\end{equation}
defines a global continuous semiflow \(\{S_R(t)\}_{t\ge0}\) on \(\mathcal X\).
\end{theorem}

The proof of Theorem~\ref{thm:hnrd_wellposedness} is provided in Appendix~\ref{app:Theorem6}. This theorem establishes that the proposed reaction--diffusion equation is dynamically well defined, so the subsequent long-time analysis can be carried out on its induced semiflow.

We now show that the proposed reaction term prevents the transverse component from collapsing. The following theorem shows that the reaction--diffusion dynamics drives \(s(t)\) toward the learnable nonzero level \(\tau_\eta\).

\begin{theorem}
\label{thm:hnrd_transverse_stabilization}
Under the conditions of Lemma~\ref{lem:normalized_diagonal_modulation}, let \(\{\mathbf S_R(t)\}_{t\ge0}\) be the global continuous semiflow generated by Eq.~\eqref{eq:hnrd_equation}. Assume that \(\mathbf Q_\phi\mathbf X_0\neq0\). Then the transverse energy \(\|\mathbf Q_\phi\mathbf S_R(t)\mathbf X_0\|_F^2\)
satisfies
\begin{equation}
\lim_{t\to\infty}
\|\mathbf Q_\phi\mathbf S_R(t)\mathbf X_0\|_F^2
=\lim_{t\to\infty}\operatorname{dist}^2
\left(
\mathbf S_R(t)\mathbf X_0,
\mathcal S_\phi
\right)=\tau_{\eta}>0 .
\end{equation}
\end{theorem}
The proof of Theorem~\ref{thm:hnrd_transverse_stabilization} is provided in Appendix~\ref{app:theorem7}. The theorem shows that the HNRD semiflow does not collapse toward the null-mode subspace. Instead, it stabilizes the transverse energy at the learnable level \(\tau_\eta\), thereby preserving nontrivial null-mode-free variation. This behavior sharply contrasts with pure diffusion, whose transverse energy decays to zero.

The stabilization of \(\|\mathbf Q_\phi\mathbf X(t)\|_F^2\) further implies preservation of hypergraph-structured variation. Indeed, since the Dirichlet energy is equivalent to the squared transverse norm on \(\mathcal S_\phi^\perp\), the non-collapse of the transverse component yields a two-sided asymptotic bound for \(\mathcal E_\phi(\mathbf X(t))\).
\begin{theorem}
\label{thm:hnrd_dirichlet_noncollapse}
Under the conditions of Lemma~\ref{lem:normalized_diagonal_modulation}, let \(\{\mathbf S_R(t)\}_{t\ge0}\) be the global continuous semiflow generated by Eq.~\eqref{eq:hnrd_equation}. Assume that \(\mathbf Q_\phi\mathbf X_0\neq0\). Then the null-mode-free Dirichlet energy along the HNRD semiflow satisfies
\begin{equation}
\begin{aligned}
\frac{\lambda_{H,\max}}{2}\tau_\eta&\ge \limsup_{t\to\infty}
\mathcal E_\phi(\mathbf S_R(t)\mathbf X_0)\\
&\ge \liminf_{t\to\infty}
\mathcal E_\phi(\mathbf S_R(t)\mathbf X_0)
\ge
\frac{\lambda_{H,2}}{2}\tau_\eta,
\label{eq:dirichlet_noncollapse_lower_bound}
\end{aligned}
\end{equation}
where \(
\lambda_{H,2}
:=
\min_{\mathbf Z\in\mathcal S_\phi^\perp,\ \|\mathbf Z\|_F=1}
\left\langle
\mathbf Z,\mathbf L_H\mathbf Z
\right\rangle_F
> 0,
\) and \(\lambda_{H,\max}
:=
\max_{\mathbf Z\in\mathcal S_\phi^\perp,\ \|\mathbf Z\|_F=1}
\left\langle
\mathbf Z,\mathbf L_H\mathbf Z
\right\rangle_F .\)
\end{theorem}

The proof of Theorem~\ref{thm:hnrd_dirichlet_noncollapse} is provided in Appendix~\ref{app:theorem8}. This theorem shows that the proposed HNRD dynamics prevents the Dirichlet energy from decaying to zero. Therefore, the node representations do not collapse to the null-mode subspace in the long-time regime.

\subsection{Discrete HNRD Layer}

The continuous-time formulation also induces a discrete hypergraph neural layer through forward Euler discretization. Given a step size \(h>0\), Eq.~\eqref{eq:hnrd_equation} yields
\begin{equation}
\mathbf X^{k+1}
=
\mathbf X^k
+
h
\left[\mathcal F_\theta(\mathbf X^k)
+
\mathcal R_\eta(\mathbf X^k)
\right].
\label{eq:discrete_hnrd_layer}
\end{equation}
This layer can be interpreted as a learnable diffusion step corrected by a bounded reaction feedback. The diffusion term performs hypergraph-structured smoothing, while the reaction term preserves nontrivial transverse variation and prevents the progressive collapse of node representations in deep propagation. We next give a simple step-size stability condition for the discrete layer. 
\begin{theorem}
\label{thm:discrete_hnrd_stepsize_stability}
Consider the discrete HNRD layer in Eq.~\eqref{eq:discrete_hnrd_layer}. Assume that \(\mathbf Q_\phi\mathbf X^0\neq0\). If the step size satisfies
\begin{equation}
0<h<1,
%\frac{2}{1+\frac14\lambda_{\max}^2(\mathbf L_H)},
\label{eq:discrete_hnrd_stepsize_condition}
\end{equation}
then the transverse energy \(s_k:=\|\mathbf Q_\phi\mathbf X^k\|_F^2
\) is uniformly bounded along depth. More precisely, there exists a constant \(C_h>0\), depending on \(h\), \(\tau_\eta\), \(\lambda_{\max}(\mathbf L_H)\), and \(s_0\), such that
\begin{equation}
s_k\le C_h,
\quad
\forall k\ge0 .
\end{equation}
\end{theorem}

The proof of Theorem~\ref{thm:discrete_hnrd_stepsize_stability} is provided in Appendix~\ref{app:theorem9}. This theorem shows that the explicit Euler discretization remains stable under a normalized residual step-size condition. The upper bound depends on the largest eigenvalue of the hypergraph Laplacian, which reflects the strongest possible smoothing strength induced by the hypergraph diffusion operator. Therefore, highly connected or strongly coupled hypergraphs require a smaller step size, while normalized hypergraph operators allow a relatively larger stable range.

In practice, the step size can either be fixed as a small constant or learned under the stability constraint. For example, one may parameterize \(h=\operatorname{Sigmoid}(\xi),\) where \(\xi\) is a learnable scalar. This parameterization keeps the discrete HNRD layer within the stable residual regime while allowing the effective propagation strength to be adapted from data.

Together with the continuous-time non-collapse results, Theorem~\ref{thm:discrete_hnrd_stepsize_stability} indicates that the proposed discrete layer preserves the stabilizing behavior of the HNRD. The discrete update in Eq.~\eqref{eq:discrete_hnrd_layer} therefore provides a stable implementation of the continuous HNRD dynamics. By stacking \(K\) such layers,
\begin{equation}
\mathbf X^0=\mathbf X_0,
\quad
\mathbf X^{k+1}
=
\operatorname{HNRDLayer}(\mathbf X^k),
\quad
k=0,\ldots,K-1,
\end{equation}
one obtains a deep hypergraph neural architecture that combines adaptive hypergraph diffusion with transverse-energy regulation. The final representation \(\mathbf X^K\) can then be passed to a task-specific readout or classifier.

\section{Experiments}
\subsection{Results on Benchmark Datasets}
\label{benchmark}
\begin{table*}[htbp]
  \centering
  \caption{Dataset Statistics Summary\label{tab:dataset_stats}}
  \setlength{\tabcolsep}{4pt}
  \begin{tabular}{l *{12}{S[table-format=5.0]}}
    \toprule
    {\textbf{Metric}} & {\textbf{Cora}} & {\textbf{Citeseer}} & {\textbf{Pubmed}} & {\textbf{Cora-CA}} & {\textbf{DBLP-CA}} & {\textbf{Zoo}}  & {\textbf{NTU2012}} & {\textbf{ModelNet40}} &{\textbf{Walmart}}&  {\textbf{Senate}} & {\textbf{House}}  \\
    \midrule
    $|V|$ & 2708 & 3312 & 19177 & 2708 & 41302 & 101 & 2012 & 12311 &88860& 282& 1290  \\
    $|E|$ & 1579 & 1079 & 7963 & 1072 & 22363 & 43 &  2012 & 12311 &69906& 315 & 340 \\
    \#features
    & 1433 & 3703 & 500 & 1433 & 1425 & 16 & 100  & 100 &100& 100 & 100  \\
    \#classes & 7 & 6 & 3 & 7 & 6 & 7  &  67 & 40 &11& 2 & 2\\
    \bottomrule
  \end{tabular}
  \vspace{0.2cm}
\end{table*}

\begin{table*}[ht]
\belowrulesep=0pt
\aboverulesep=0pt
\renewcommand{\arraystretch}{1.5}
\caption{Performance comparison on academic hypergraph datasets (Mean accuracy (\%) ± standard deviation), with the best result highlighted in \colorbox[RGB]{255,190,110}{\textbf{bold}} and the second-best result \colorbox[RGB]{255,223,175}{\underline{underlined}}.}
\label{tab:academic_results}
\centering
\begin{tabular}{|c|c|ccccc|c|}
\toprule
& \textbf{Models}& \textbf{Cora} & \textbf{Citeseer} & \textbf{Pubmed}&\textbf{Cora-CA}&\textbf{DBLP-CA} & \textbf{Rank}$\downarrow$\\ 
\midrule  
\multirow[c]{4}{*}{Graph Diffusion Models} & GRAND & {{78.85$\pm$1.56}} & 72.93$\pm$1.23 & {{86.19$\pm$0.35}} & {{82.02$\pm$1.21}} & {{88.82$\pm$0.45}} &15\\
& GRAND++ & 76.97$\pm$1.53 & 73.75$\pm$0.40 & 85.41$\pm$0.65 & 81.47$\pm$2.54 & 88.89$\pm$0.24 &16\\
& GREAD & 78.24$\pm$2.22 & 74.32$\pm$2.94 & 87.28$\pm$0.49 & 81.56$\pm$1.23 & 87.43$\pm$0.19 &14\\
& RDGNN & 76.23$\pm$2.36 & 72.98$\pm$2.17 & 85.68$\pm$1.78 & 78.29$\pm$1.17 & 88.92$\pm$0.62 &19\\
\midrule 
\multirow[c]{5}{*}{Standard HGNNs} & HGNN & 79.39$\pm$1.36 & 72.45$\pm$1.16 & 86.44$\pm$0.44 & 82.64$\pm$1.65 & 91.03$\pm$0.20 & 10\\   
& HyperGCN & 78.45$\pm$1.26 & 71.28$\pm$0.82 & 82.84$\pm$8.67 & 79.48$\pm$2.08 & 89.38$\pm$0.25 &18\\
& HCHA& 79.14$\pm$1.02 & 72.42$\pm$1.42 & 86.41$\pm$0.36 & 82.55$\pm$0.97 & 90.92$\pm$0.22 &12\\
& HNHN & 76.36$\pm$1.92 & 72.64$\pm$1.57 & 86.90$\pm$0.30 & 77.19$\pm$1.49 & 86.78$\pm$0.29& 20\\
& UNIGCNII& 78.81$\pm$1.05 & 73.05$\pm$2.21 & 88.25$\pm$0.40 & 83.60$\pm$1.14 & 91.69$\pm$0.19 &8\\
\midrule 
\multirow[c]{5}{*}{Expressive HGNNs} & AllSetTransformer& 78.59$\pm$1.47 & 73.08$\pm$1.20 & 88.72$\pm$0.37 & 83.63$\pm$1.47 & 91.53$\pm$0.23 & 7\\
& AllDeepSets & 76.88$\pm$1.80 & 70.83$\pm$1.63 & {88.75$\pm$0.33 }& 81.97$\pm$1.50 & 91.27$\pm$0.27 & 13\\
& ED-HNN & 80.31$\pm$1.35 & 73.70$\pm$1.38 & 
\cellcolor[RGB]{255,190,110}\textbf{89.03$\pm$0.53} & 83.97$\pm$1.55 & 
{\cellcolor[RGB]{255,223,175}\underline{91.90$\pm$0.19}} &4\\
& HyperGINE & 79.26$\pm$0.41 & 73.72$\pm$0.52 & 87.91$\pm$0.28 & 82.88$\pm$0.48 & 88.18$\pm$0.21 &9\\
& KHGNN& 80.67$\pm$0.76 & 74.80$\pm$1.10& 88.47$\pm$0.47 & 84.25$\pm$0.74 & 89.64$\pm$0.32 &6\\
\midrule 
\multirow[c]{5}{*}{Deep/Stable HGNNs}& Deep-HGCN & 78.04$\pm$1.17 & 72.78$\pm$1.81 & 85.49$\pm$0.29 & 82.64$\pm$2.32 &87.04$\pm$0.27  &17\\
& Implicit HNN & 77.01$\pm$1.65 & 74.51$\pm$1.45 & 85.93$\pm$0.36 & 83.67$\pm$1.64 & 89.47$\pm$0.52 &11\\
& FrameHGNN & 81.51$\pm$0.99 & 74.72$\pm$2.10 &88.73$\pm$0.42 & 85.18$\pm$0.69 & 89.37$\pm$0.35 &5\\
& HND &
81.76$\pm$1.12 & 75.51$\pm$1.37 & 88.52$\pm$0.30 &
{\cellcolor[RGB]{255,223,175}\underline{85.49$\pm$1.05}} & 91.71$\pm$0.25 &3\\
& RFHND & 
{\cellcolor[RGB]{255,223,175}\underline{81.83$\pm$1.48}} & \cellcolor[RGB]{255,190,110}\textbf{76.04$\pm$1.43} & 88.88$\pm$0.35&
85.40$\pm$1.05 & 91.20$\pm$0.41& \cellcolor[RGB]{255,223,175}\underline{2}\\
\midrule
Ours& HNRD & \cellcolor[RGB]{255,190,110}\textbf{81.97$\pm$1.07} & \cellcolor[RGB]{255,223,175}\underline{75.85}$\pm$1.25 &\cellcolor[RGB]{255,223,175}\underline{88.97$\pm$0.17}  & \cellcolor[RGB]{255,190,110}\textbf{85.73$\pm$0.88} & \cellcolor[RGB]{255,190,110}\textbf{92.11$\pm$0.12} & \cellcolor[RGB]{255,190,110}\textbf{1}\\
\bottomrule
\end{tabular}
\end{table*}

\begin{table*}[ht]
\belowrulesep=0pt
\aboverulesep=0pt
\renewcommand{\arraystretch}{1.5}
\caption{Performance comparison on real-world hypergraph datasets (Mean accuracy (\%) ± standard deviation), with the best result highlighted in \colorbox[RGB]{255,190,110}{\textbf{bold}} and the second-best result \colorbox[RGB]{255,223,175}{\underline{underlined}}.}
\label{tab:realworld_results}
\centering
\begin{tabular}{|c|c|cccccc|c|}
\toprule
& \textbf{Models}&  \textbf{Zoo}&\textbf{ NTU2012 }& \textbf{ModelNet40}&\textbf{Walmart}
& \textbf{Senate}&\textbf{House}&\textbf{Rank}$\downarrow$\\
\midrule  
\multirow[c]{4}{*}{Graph Diffusion Models} & GRAND & {{86.73$\pm$6.17}} & 84.11$\pm$2.05 & {{97.21$\pm$0.89}} & 75.97$\pm$0.27 & {{61.13$\pm$6.72}} & {{70.02$\pm$1.85}}&16\\
& GRAND++ & 96.30$\pm$3.65 & 87.52$\pm$2.06  & 95.79$\pm$0.44 & 75.95$\pm$0.48 & 58.92$\pm$3.76 & 68.77$\pm$2.79 &17\\
& GREAD & 97.92$\pm$2.20 & 89.39$\pm$0.48 & 96.93$\pm$0.36 & 77.64$\pm$0.22 & 61.11$\pm$4.62 & 70.06$\pm$3.53 &9\\
& RDGNN & 97.38$\pm$2.13 & 87.50$\pm$1.68 & 95.84$\pm$0.61 & 74.78$\pm$0.62 & 63.06$\pm$5.40 & 65.19$\pm$3.24 &15\\
\midrule 
\multirow[c]{5}{*}{Standard HGNNs} & HGNN & 92.50$\pm$4.58 & 87.72$\pm$1.35 & 95.44$\pm$0.33 &80.33$\pm$0.42& 48.59$\pm$4.52&61.39$\pm$2.96&18\\   
& HyperGCN & 93.62$\pm$4.18 & 56.36$\pm$4.86 & 75.89$\pm$5.26&81.05$\pm$0.59 & 42.45$\pm$3.67&48.32$\pm$2.93&20\\
& HCHA & 93.65$\pm$6.15 & 87.48$\pm$1.87& 94.48$\pm$0.28&80.33$\pm$0.80 & 48.62$\pm$4.41&61.36$\pm$2.53&19\\
& HNHN & 93.59$\pm$5.88 & 89.11$\pm$1.44&97.84$\pm$1.25&81.35$\pm$0.61& 50.93$\pm$6.33&67.80$\pm$2.59 & 12\\
& UNIGCNII & 93.65$\pm$4.37 & 89.30$\pm$1.33 & 98.07$\pm$0.23&81.12$\pm$0.67 & 49.30$\pm$4.25&61.70$\pm$3.37&11\\
\midrule 
\multirow[c]{5}{*}{Expressive HGNNs} & AllSetTransformer & 97.50$\pm$3.59& 88.69$\pm$1.24&98.20$\pm$0.20&81.38$\pm$0.58& 51.83$\pm$5.22 &69.33$\pm$2.20& 6\\
& AllDeepSets & 95.39$\pm$4.77 & 88.09$\pm$1.52&96.98$\pm$0.26&81.06$\pm$0.54& 48.17$\pm$5.67&67.82$\pm$2.40& 14\\
& ED-HNN & 97.57$\pm$7.22  & 88.67$\pm$0.92 & 97.83$\pm$0.33 & 79.98$\pm$0.84 & 64.79$\pm$5.14&72.45$\pm$2.28&8\\
& HyperGINE & 95.68$\pm$8.44  & 88.52$\pm$0.42 & 97.61$\pm$0.16 & 79.47$\pm$0.89 & 58.33$\pm$6.84 & 69.48$\pm$3.37 &13\\
& KHGNN & 97.03$\pm$5.47  & 89.60$\pm$1.64 & 98.33$\pm$0.14 & 81.69$\pm$0.52 & 59.81$\pm$4.92 & 68.53$\pm$2.69 &5\\
\midrule 
\multirow[c]{5}{*}{Deep/Stable HGNNs}& Deep-HGCN & 97.32$\pm$5.31 & 87.87$\pm$1.76 & 98.09$\pm$0.55 & 81.80$\pm$0.73 & 58.73$\pm$5.01 & 68.89$\pm$1.83 &10\\
& Implicit HNN & 97.79$\pm$6.09 & 88.75$\pm$1.24 & 97.96$\pm$0.53 & 80.84$\pm$0.45 & 66.40$\pm$4.11 & 67.22$\pm$1.52 &7\\
& FrameHGNN &97.30$\pm$3.79 &89.98$\pm$2.02& 98.41$\pm$0.18 & 80.46$\pm$0.65 &67.61$\pm$5.27& 72.82$\pm$2.22 &4\\
& HND & 
\cellcolor[RGB]{255,223,175}\underline{99.19$\pm$1.24} & 93.32$\pm$0.99 & 
93.32$\pm$0.99 & \cellcolor[RGB]{255,223,175}\underline{82.69$\pm$0.86} & \cellcolor[RGB]{255,223,175}\underline{70.00$\pm$4.44} & \cellcolor[RGB]{255,223,175}\underline{73.69$\pm$2.30}&3\\
& RFHND & 
97.70$\pm$2.14 & 
\cellcolor[RGB]{255,190,110}\textbf{98.48$\pm$0.20} & \cellcolor[RGB]{255,223,175}\underline{98.56$\pm$0.16}& 
82.40$\pm$1.40 & 68.12$\pm$4.98 & 73.52$\pm$2.25&\cellcolor[RGB]{255,223,175}\underline{2}\\
\midrule
& HNRD & \cellcolor[RGB]{255,190,110}\textbf{99.39$\pm$1.97} & \cellcolor[RGB]{255,223,175}\underline{97.98$\pm$0.47} & \cellcolor[RGB]{255,190,110}\textbf{98.74$\pm$0.22} & \cellcolor[RGB]{255,190,110}\textbf{83.35$\pm$0.65} & \cellcolor[RGB]{255,190,110}\textbf{70.28$\pm$4.59} & \cellcolor[RGB]{255,190,110}\textbf{75.86$\pm$1.85} &\cellcolor[RGB]{255,190,110}\textbf{1}\\
\bottomrule
\end{tabular}
\end{table*}
\textbf{Datasets.} We evaluate HNRD on a diverse set of hypergraph benchmarks, spanning academic citation/co-authorship networks and real-world applications.

For academic evaluation, we adopt five standard hypergraph benchmarks: Cora, Citeseer, and Pubmed, along with Cora-CA and DBLP-CA~\cite{yadati2019hypergcn}. Each node is a paper with bag-of-words features, labeled by research topics. These datasets cover representative citation and co-authorship relations commonly used in hypergraph learning. To assess generalization beyond such graphs, we further include datasets with distinct structural properties: Zoo from the UCI repository~\cite{dua2017uci} with mixed categorical attributes; ModelNet40~\cite{wu20153d} and NTU2012~\cite{chen2003visual} from 3D vision; Walmart~\cite{amburg2020clustering} capturing transaction records; and House~\cite{chodrow2021generative} and Senate~\cite{fowler2006connecting} representing legislative interactions. We follow standard hypergraph construction protocols and split each dataset into train/validation/test sets (50\%/25\%/25\%). All results are averaged over 20 independent random splits. For datasets without intrinsic node features, we adopt Gaussian random vectors as attributes. Detailed dataset statistics are summarized in Table~\ref{tab:dataset_stats}.

\textbf{Baselines.} To comprehensively evaluate HNRD, we compare it with representative baselines spanning graph neural diffusion models, conventional hypergraph neural networks, expressive higher-order architectures, and deep anti-oversmoothing models. Since HNRD is directly motivated by diffusion and reaction–diffusion dynamics, we first include graph-based diffusion models: GRAND~\cite{chamberlain2021grand}, GRAND++~\cite{thorpe2022grand++}, GREAD~\cite{choi2023gread}, and RDGNN~\cite{eliasof2024graph}. For these graph-based baselines, we convert hypergraphs to graphs using clique expansion strategies. Standard hypergraph baselines include spectral and message-passing models: HGNN~\cite{feng2019hypergraph}, HyperGCN~\cite{yadati2019hypergcn}, HCHA~\cite{bai2021hypergraph}, HNHN~\cite{dong2020hnhn}, and UniGCNII~\cite{huang2021unignn}. We further compare with models emphasizing representation power, including AllSetTransformer and AllDeepSets~\cite{chien2021you}, ED-HNN~\cite{wang2022equivariant}, HyperGINE and KHGNN~\cite{xie2025k}. To assess performance under deep propagation, we include several baselines specifically related to stability and oversmoothing mitigation: Deep-HGCN~\cite{chen2022preventing}, Implicit HNN~\cite{li2025implicit}, FrameHGNN~\cite{li2025deep}, HND~\cite{zhou2026hypergraph}, and RFHND~\cite{zhou2026tackling}. All models are implemented via PyTorch Geometric~\cite{fey2019fast} or official code, using identical splits and evaluation protocols.

\textbf{Experimental Settings.} We train HNRD using Adam and tune its hyperparameters via validation performance on each dataset. The search space includes learning rate, weight decay, input dropout, hidden dropout, hidden dimension, reaction--diffusion step size, and number of propagation layers. Specifically, we search learning rate over \{0.001,0.01\}, weight decay over \{0.001,0.01,0.05\}, and input dropout over \{0.05,0.1,0.2\}. Hidden dimension is chosen from \{16,32,64,128\}, and the number of propagation layers from {3,4}. For the reaction--diffusion step size, we search over \{0.1,0.3,0.5\}, and additionally parameterize it as a learnable scalar with a sigmoid constraint for numerical stability. All experiments are run on a single NVIDIA RTX A4000 GPU with a fixed random seed to ensure reproducibility. Source code is available at \href{https://github.com/CASZhouzhiheng/HNRD}{https://github.com/CASZhouzhiheng/HNRD}.

\textbf{Results.} The benchmark classification results are reported in Tables~\ref{tab:academic_results} and~\ref{tab:realworld_results}, where the final rank is computed by averaging the per-dataset ranks. HNRD achieves the best overall rank on both academic and real-world benchmarks. It ranks first on 3 out of 5 academic datasets, i.e., Cora, Cora-CA, and DBLP-CA, and second on Citeseer and Pubmed. On real-world datasets, it ranks first on 5 out of 6 datasets, including Zoo, ModelNet40, Walmart, Senate, and House, and second on NTU2012. Overall, HNRD obtains the best result on 8 out of 11 datasets and remains second-best on the other 3, demonstrating stable performance across diverse hypergraph domains. Compared with the strongest competing method on each dataset, HNRD achieves its largest gain of 2.17\% on House. Notably, HND and RFHND are strong hypergraph diffusion baselines tailored for stability and oversmoothing mitigation. Even compared with RFHND, the strongest overall competitor in terms of average rank, HNRD improves the average accuracy by 0.26\% on academic datasets and 1.14\% on real-world datasets. The comparison with graph-based diffusion models further highlights the advantage of incidence-level hypergraph dynamics. GRAND, GRAND++, GREAD, and RDGNN generally underperform hypergraph-specific methods, suggesting that pairwise graph expansions cannot fully preserve higher-order mixing patterns. Moreover, unlike HND and RFHND, which regulate diffusion without explicitly compensating transverse dissipation, HNRD introduces a reaction term that directly counteracts diffusion-induced collapse and stabilizes node-discriminative representations.

\subsection{Results on Synthetic Heterophilic Hypergraph Dataset}
\begin{table*}[htbp]
% \footnotesize
\centering
\caption{Model performance on synthetic hypergraphs with controlled heterophily $\alpha$.}
% \rowcolors{1}{}{blue!10} % 设置标题行背景色
\label{tab:synthetic_heterophily}
\begin{tabular}{cccccccc}
\toprule
\textbf{Model}& \bm{$\alpha = 1$} & \bm{$\alpha = 2$} & \bm{$\alpha = 3$} &\bm{$\alpha = 4$}& \bm{$\alpha = 5$} &\bm {$\alpha = 6$} &\bm {$\alpha = 7$} \\
\midrule
HGNN & 92.58 $\pm$ 0.57 & 87.40 $\pm$ 0.89 & 82.32 $\pm$ 0.84& 77.18 $\pm$ 1.08 & 70.06 $\pm$ 1.35 & 69.98 $\pm$ 0.85 & 68.41 $\pm$ 1.39 \\
HyperGCN  & 83.40 $\pm$ 1.67 & 78.13 $\pm$ 1.57 & 77.70 $\pm$ 1.01& 74.90 $\pm$ 1.31 & 50.15 $\pm$ 1.94 & 52.11 $\pm$ 0.75 & 49.40 $\pm$ 1.63 \\
UNIGCN    & 90.42 $\pm$ 1.04 & 82.12 $\pm$ 0.67 & 80.04 $\pm$ 0.98& 76.86 $\pm$ 1.10 & 71.50 $\pm$ 1.50 & 52.30 $\pm$ 1.18 & 49.20 $\pm$ 0.74 \\
ED-HNN     & 93.88 $\pm$ 0.67 & 89.89 $\pm$ 0.59 & 83.52 $\pm$ 0.85& 76.65 $\pm$ 1.21 & 72.18 $\pm$ 1.36 & 52.98 $\pm$ 0.78 & 49.65 $\pm$ 1.16 \\
KHGNN & 94.93 $\pm$ 0.38 & 90.05 $\pm$ 0.24 & 84.27 $\pm$ 0.96 & 79.37 $\pm$ 0.92 &73.26 $\pm$ 0.98  & 64.09 $\pm$ 0.77 & 53.78 $\pm$ 1.43 \\
Deep-HGCN & 95.04 $\pm$ 0.73 & 89.70 $\pm$ 0.63 & 83.04 $\pm$ 0.61 & 78.67 $\pm$ 1.11 & 72.94 $\pm$ 2.05 & 54.54 $\pm$ 1.08 & 53.26 $\pm$ 1.91 \\
Implicit HNN & 95.00 $\pm$ 0.24 & 88.69 $\pm$ 1.92 & 82.74 $\pm$ 1.29 & 77.67 $\pm$ 1.98 & 74.07 $\pm$ 1.34 & 65.98 $\pm$ 1.39 & 63.22 $\pm$ 1.38 \\
FrameHGNN & 95.11 $\pm$ 0.59 & 90.06 $\pm$ 0.64 & 84.08 $\pm$ 0.73 & 79.54 $\pm$ 0.75 & 74.39 $\pm$ 1.48 & 72.55 $\pm$ 0.87 & 70.40 $\pm$ 1.21 \\
HND & 95.34 $\pm$ 0.44 & \underline{90.50 $\pm$ 0.68} &\underline{84.33 $\pm$ 0.90}  & \underline{80.32 $\pm$ 0.92} & \underline{75.67 $\pm$ 1.16} &\underline{73.74 $\pm$ 0.94} & 73.28 $\pm$ 0.96 \\
RFHND   & \underline{95.63} $\pm$ 0.62 &90.17 $\pm$ 0.67 &84.21 $\pm$ 0.81& 77.48 $\pm$ 0.98 &74.26 $\pm$ 1.57 &73.51 $\pm$ 1.07 &\underline{73.46 $\pm$ 1.35}\\
\midrule 
HNRD  & \textbf{96.19 $\pm$ 0.29} & \textbf{90.85 $\pm$ 0.51} & \textbf{84.70 $\pm$ 0.55} & \textbf{80.86 $\pm$ 1.06} & \textbf{76.28 $\pm$ 1.11} & \textbf{74.37 $\pm$ 0.67} & \textbf{73.98 $\pm$ 1.19} \\
\bottomrule
\end{tabular}
\end{table*}
\textbf{Experimental Settings.}
We construct synthetic hypergraph datasets to evaluate the robustness of HNRD under controllable structural heterophily. The data generation follows the contextual hypergraph stochastic block model~\cite{ghoshdastidar2014consistency,deshpande2018contextual,chien2018community}. Each dataset contains two equally sized node classes, with 2,500 nodes per class. We randomly generate 1,000 hyperedges, each containing 15 nodes.

The heterophily level is controlled by the class composition of each hyperedge. For a hyperedge associated with class $i$, we sample $15-\alpha_i$ nodes from class $i$ and $\alpha_i$ nodes from the other class. We define the dataset-level heterophily level as \(\alpha=\min\{\alpha_1,\alpha_2\}.\) Smaller values of $\alpha$ indicate stronger homophily while larger values indicate stronger heterophily. In our experiments, we consider homophilic settings with $\alpha=1,2,3$ and heterophilic settings with $\alpha=4,\ldots,7$. We generate node attributes from label-conditioned Gaussian distributions with standard deviation $\sigma=1.0$, and compare HNRD against ten representative hypergraph baselines. Following the same protocol as in Section~\ref{benchmark}, we split nodes into train/validation/test sets (50\%/25\%/25\%) and report averaged results over 10 random splits for statistical robustness.

\textbf{Results.}
The results on synthetic heterophilic hypergraphs are reported in Table~\ref{tab:synthetic_heterophily}. HNRD achieves the best performance across all heterophily levels, from the strongly homophilic setting $\alpha=1$ to the highly heterophilic setting $\alpha=7$. As $\alpha$ increases, the class composition within each hyperedge becomes more mixed, and most baselines suffer a clear performance degradation. This trend is especially evident for HyperGCN, UNIGCN, and ED-HNN, whose accuracy drops sharply under strong heterophily. Compared with strong diffusion-based hypergraph baselines such as HND and RFHND, HNRD consistently obtains higher accuracy at every heterophily level. The improvement remains stable in both homophilic settings ($\alpha=1,2,3$) and heterophilic settings ($\alpha=4,\ldots,7$). In particular, under stronger heterophily, HNRD achieves $76.28\%$, $74.37\%$, and $73.98\%$ accuracy for $\alpha=5,6,7$, respectively, outperforming HND and RFHND in all three cases. These results indicate that the proposed reaction--diffusion mechanism is effective not only for mitigating oversmoothing, but also for preserving discriminative representations when hyperedges contain mixed-class nodes.

\subsection{Oversmoothing Analysis}
\begin{figure*}[ht]
\centering
{\includegraphics[scale=0.54]{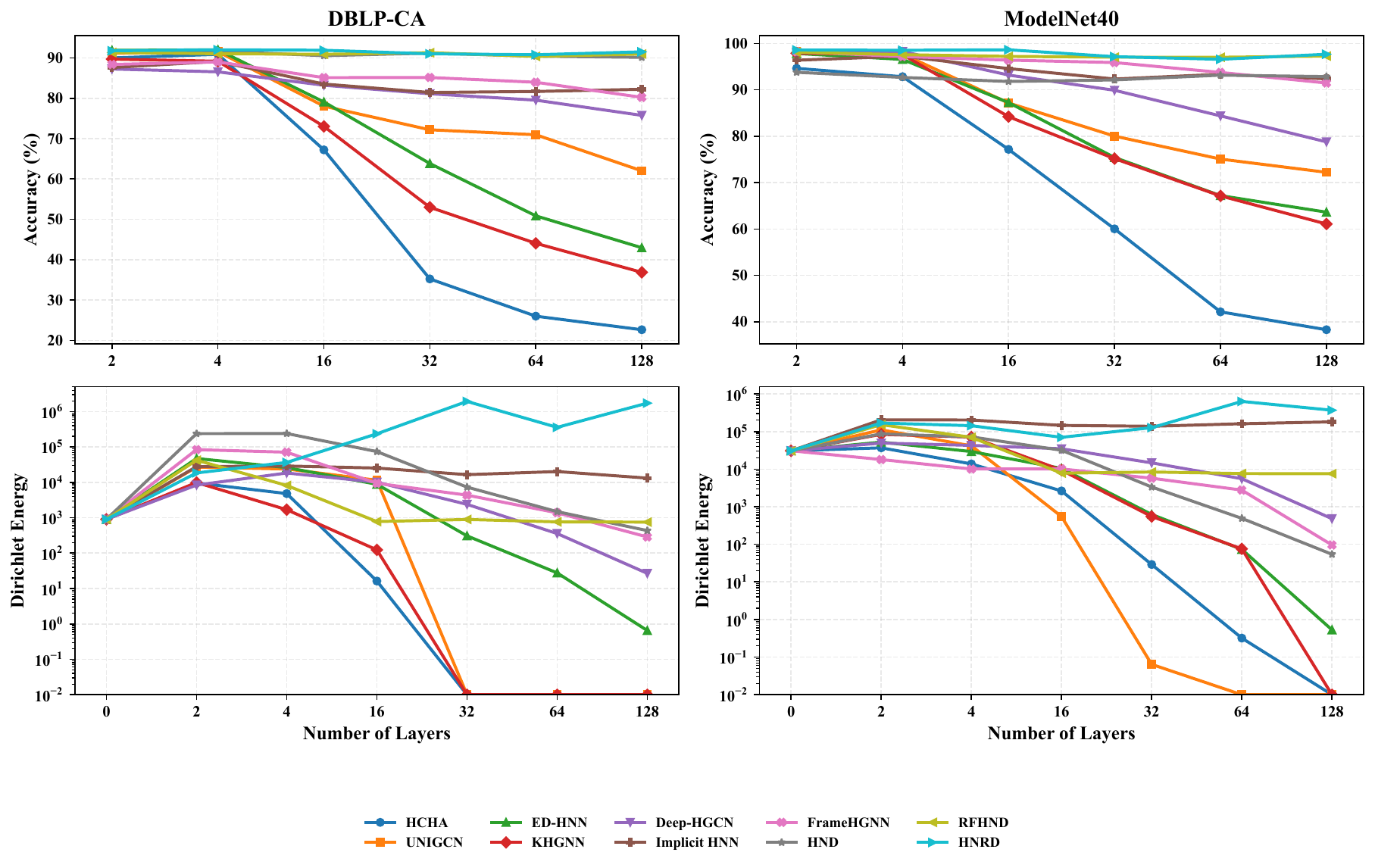}}
\caption{Depth-wise oversmoothing analysis on DBLP-CA and ModelNet40.}
\label{fig:oversmoothing}
\end{figure*}
The depth-wise oversmoothing analysis is shown in Fig.~\ref{fig:oversmoothing}. As the number of layers increases, several baseline models suffer from clear performance degradation, especially HCHA, KHGNN, and ED-HNN, whose accuracy drops substantially under deep propagation. In contrast, HNRD maintains consistently high accuracy on both DBLP-CA and ModelNet40 even when the depth increases to 128 layers, demonstrating strong robustness against deep propagation. Notably, HNRD also remains more stable than recent deep or stable hypergraph baselines, indicating that the proposed reaction--diffusion design is effective under very deep propagation regimes.

The Dirichlet energy curves further explain this behavior. For many baselines, the Dirichlet energy rapidly decreases as the depth increases, indicating that node representations progressively lose nontrivial structural variation. By contrast, HNRD preserves a stable and nonzero Dirichlet energy across depths. This observation is consistent with the proposed reaction--diffusion mechanism, which explicitly counteracts transverse diffusion dissipation and prevents representation collapse. These results jointly support our theoretical analysis, showing that HNRD mitigates oversmoothing by preserving null-mode-free structural variation during deep propagation.

\subsection{Ablation Study}
\begin{table}[h!]
\centering
\caption{Ablation results of HNRD on three representative datasets.}
\label{tab:ablation}
\begin{tabular}{lccc}
\toprule
\textbf{Variants} & \textbf{DBLP-CA} & \textbf{ModelNet40} & \textbf{Walmart} \\
\midrule 
HND & 91.71 $\pm$ 0.25 & 93.32 $\pm$ 0.99 & 82.69 $\pm$ 0.86 \\
w/o Compensation & 91.65 $\pm$ 0.21 & 97.99 $\pm$ 0.42 & 82.57 $\pm$ 0.46 \\
w/o Feedback & 91.89 $\pm$ 0.14 & 98.04 $\pm$ 0.19  &  82.89 $\pm$ 0.72 \\
\midrule 
\textbf{HNRD} & \textbf{92.11 $\pm$ 0.12} & \textbf{98.74 $\pm$ 0.22} & \textbf{83.35 $\pm$ 0.65} \\
\bottomrule
\end{tabular}
\end{table}
To assess the contribution of each key component in HNRD, we conduct ablation experiments by comparing the full model with several variants. Since the proposed reaction--diffusion mechanism consists of an instantaneous transverse-dissipation compensation term and a bounded feedback term, we consider the following settings:
\begin{itemize}
    \item \textbf{HND / Pure Diffusion:} 
    In this variant, the entire reaction term is removed, and the model reduces to pure hypergraph neural diffusion. The update is therefore driven only by the incidence-level diffusion operator.

    \item \textbf{w/o Compensation:} 
    In this variant, the instantaneous compensation term $R_\theta(X)Q_\phi X$ is removed from the reaction function. The model only keeps the bounded feedback term $\tanh(\tau_\eta-\|Q_\phi X\|_F^2)Q_\phi X$.

    \item \textbf{w/o Feedback:} 
    In this variant, the bounded feedback term 
    $\tanh(\tau_\eta-\|Q_\phi X\|_F^2)Q_\phi X$ is removed, while the compensation term $R_\theta(X)Q_\phi X$ is retained.
\end{itemize}

\textbf{Results.}
The ablation results are reported in Table~\ref{tab:ablation}. HNRD achieves the best performance on all three datasets, confirming the effectiveness of the complete reaction--diffusion design. Removing the entire reaction term, i.e., reducing the model to HND, leads to clear performance degradation, especially on ModelNet40, where the accuracy drops from $98.74\%$ to $93.32\%$. This verifies that pure hypergraph diffusion alone is insufficient to prevent representation degradation. Among the two reaction components, removing the compensation term causes a larger drop than removing the feedback term on all datasets. This indicates that the instantaneous compensation term plays a more fundamental role in counteracting diffusion-induced transverse dissipation, while the bounded feedback term further stabilizes the transverse energy. Overall, both components contribute to performance improvement, and their combination yields the strongest results.

\subsection{Parameter Analysis}
\begin{figure}[ht]
\centering
{\includegraphics[scale=0.39]{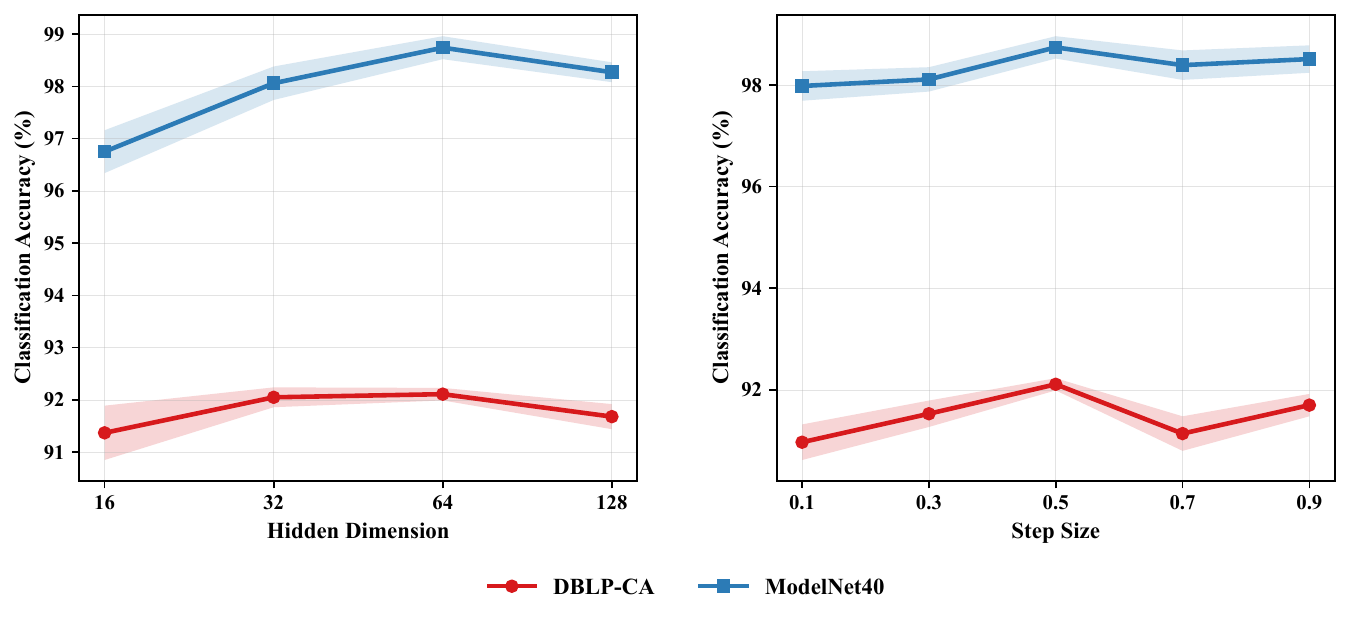}}
\caption{Parameter analysis of HNRD with respect to hidden dimension and step size on DBLP-CA and ModelNet40.}
\label{fig:parameter}
\end{figure}
Fig.~\ref{fig:parameter} presents the sensitivity analysis of HNRD with respect to the hidden dimension and step size on DBLP-CA and ModelNet40. For the hidden dimension, the performance generally improves as the dimension increases from 16 to 64, indicating that a larger representation space helps capture richer higher-order information. When the hidden dimension further increases to 128, the performance slightly decreases, suggesting that an excessively large hidden space may introduce redundancy or overfitting. Overall, the best or near-best performance is achieved at hidden dimension 64.

For the step size, HNRD remains relatively stable across different values, demonstrating that the model is not overly sensitive to this hyperparameter. Both datasets achieve strong performance around step size 0.5, while very small or overly large step sizes lead to slightly inferior results. These observations indicate that a moderate reaction--diffusion step size provides a better balance between hypergraph diffusion and reaction-based representation regulation.

\subsection{Visualization Analysis}
\begin{figure*}[ht]
\centering
{\includegraphics[scale=0.42]{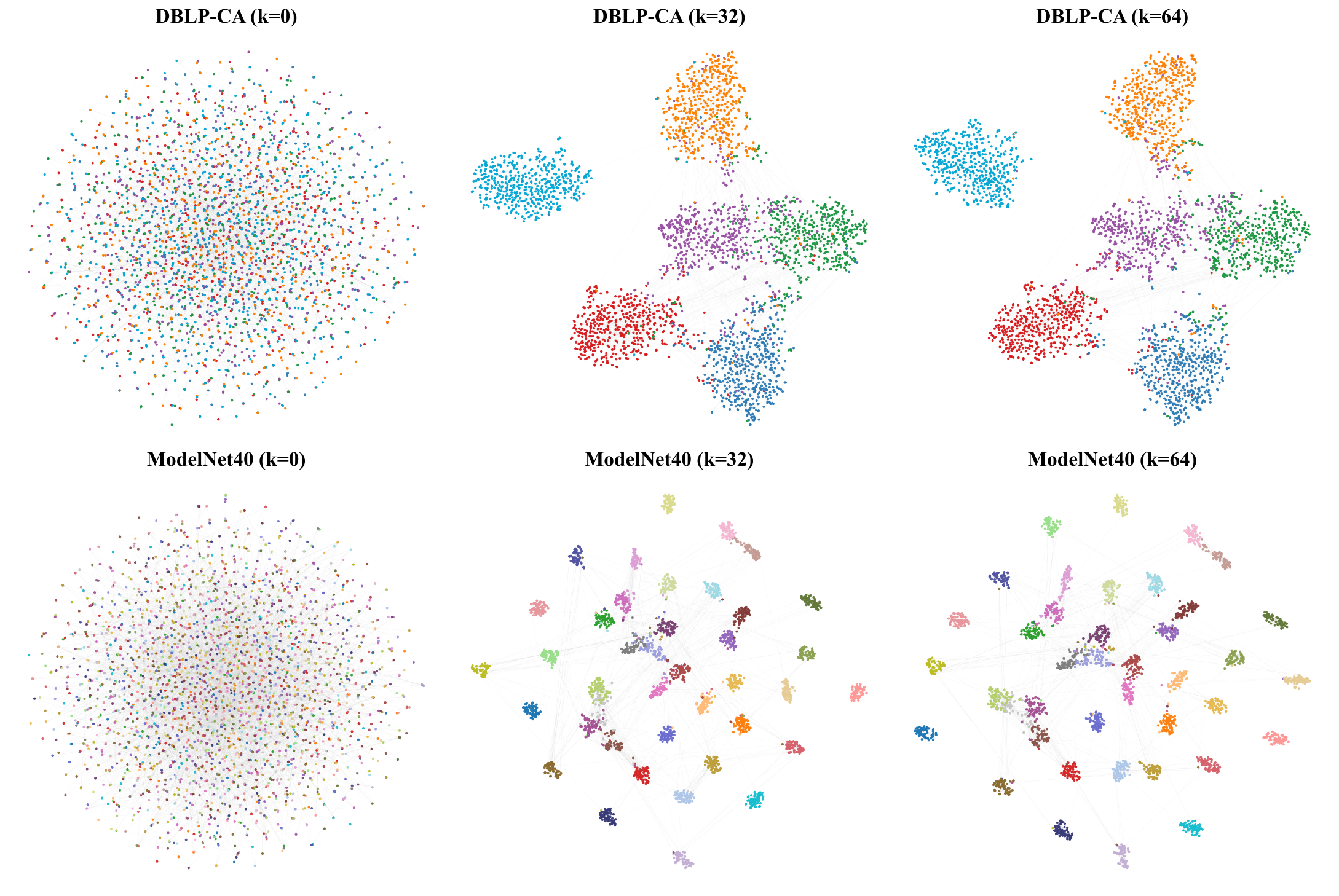}}
\caption{Visualization of node representation evolution on DBLP-CA and ModelNet40 under different propagation depths. Node embeddings are shown at $k=0$, $k=32$, and $k=64$, where colors denote node classes.}
\label{fig:visualization}
\end{figure*}
\begin{figure*}[htbp!]
\centering
{\includegraphics[scale=0.5]{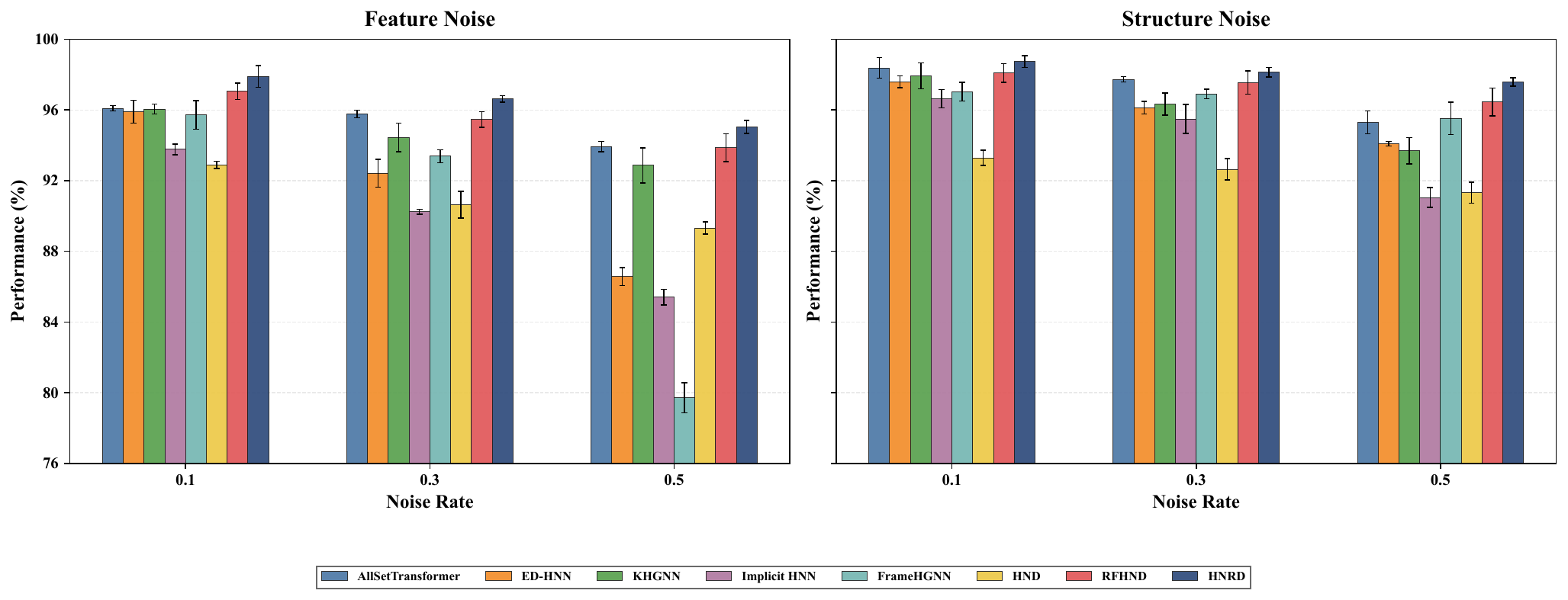}}
\caption{Robustness analysis under feature-level and structure-level perturbations on ModelNet40.}
\label{fig:robustness}
\end{figure*}
To examine the representation dynamics of HNRD under deep propagation, we visualize node embeddings on DBLP-CA and ModelNet40 at different propagation depths, i.e., $k=0$, $k=32$, and $k=64$. The high-dimensional embeddings are projected into a two-dimensional space using t-SNE, with node colors denoting class labels. As shown in Fig.~\ref{fig:visualization}, the initial embeddings at $k=0$ are highly mixed, with unclear class boundaries. As the propagation depth increases, HNRD gradually produces more compact and class-separable clusters. Even at $k=64$, the representations do not collapse into a homogeneous distribution, indicating that HNRD can maintain discriminative node-wise variation under deep propagation. This visualization qualitatively supports the ability of the proposed reaction--diffusion mechanism to alleviate oversmoothing.

\subsection{Robustness analysis}
To assess the robustness of HNRD under noisy input conditions, we conduct experiments on ModelNet40 by considering both feature-level and structure-level perturbations. For feature noise, node attributes are corrupted using mask noise that randomly sets a portion of feature values to zero. The masking ratio is controlled by the noise rate. For structure noise, we perturb the hypergraph by randomly removing existing hyperedges and inserting new hyperedges formed by randomly sampled node subsets, with the perturbation strength also controlled by the noise rate.

The robustness results on ModelNet40 are reported in Fig.~\ref{fig:robustness}. As the noise rate increases, most baseline methods suffer from clear performance degradation under both feature-level and structure-level perturbations. This indicates that noisy node attributes and disturbed hypergraph connectivity can substantially affect higher-order representation learning.

In contrast, HNRD consistently maintains stronger performance across different noise rates. The advantage is especially evident under larger perturbation levels, where the performance of several baselines drops more sharply. These results suggest that the proposed reaction--diffusion mechanism improves not only resistance to oversmoothing, but also robustness against noisy features and perturbed hypergraph structures.

\subsection{Runtime Comparison}
\begin{figure}[ht]
\centering
{\includegraphics[scale=0.4]{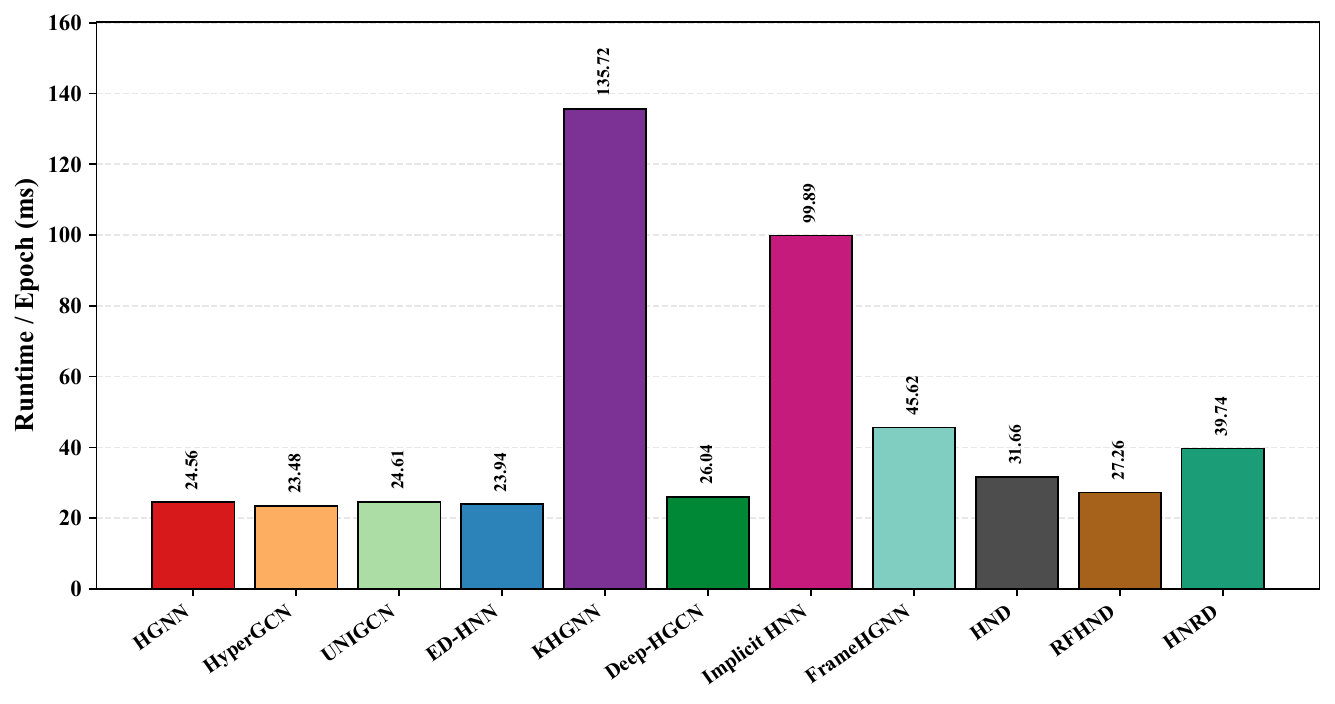}}
\caption{Runtime comparison of different HGNNs on DBLP-CA.}
\label{fig:runtime}
\end{figure}
Runtime efficiency is evaluated on DBLP-CA, as shown in Fig.~\ref{fig:runtime}. HNRD requires $39.74$ ms per epoch, incurring only a moderate overhead compared with HND and RFHND due to the additional reaction--diffusion regulation. Meanwhile, it is substantially more efficient than KHGNN and Implicit HNN, which require $135.72$ ms and $99.89$ ms per epoch, respectively, and is also faster than FrameHGNN. These results indicate that HNRD achieves a favorable balance between computational efficiency and reaction-based oversmoothing mitigation.
\section{Conclusion}
In this paper, we studied oversmoothing in HGNNs from a dynamical-systems perspective. By formulating hypergraph message passing as a learnable incidence-level diffusion equation, we showed that pure hypergraph neural diffusion intrinsically contracts the null-mode-free component of node representations and drives the associated Dirichlet energy to zero. This provides a principled explanation of oversmoothing in HGNNs as a transverse-energy dissipation process induced by higher-order diffusion. To address this issue, we proposed Hypergraph Neural Reaction--Diffusion (HNRD), which introduces a reaction term acting only on the transverse component. The reaction mechanism combines instantaneous compensation of diffusion-induced dissipation with bounded feedback toward a positive learnable energy level. We established global well-posedness, non-collapse of transverse variation, positive preservation of the null-mode-free Dirichlet energy, and a stable discrete implementation through forward Euler discretization. Extensive experiments on benchmark, real-world, and synthetic heterophilic hypergraph datasets demonstrated the effectiveness of HNRD. The model achieves strong predictive performance, remains stable under deep propagation, preserves nonzero Dirichlet energy, and shows robustness under perturbed settings with practical computational cost. These results suggest that incidence-level reaction--diffusion dynamics provide a principled route for designing deep hypergraph neural architectures that retain higher-order expressiveness while avoiding representation collapse.

\bibliographystyle{IEEEtran}
\bibliography{refer}

@article{zhou2026hypergraph,
  title={Hypergraph Neural Diffusion: A PDE-Inspired Framework for Hypergraph Message Passing},
  author={Zhou, Zhiheng and Zhou, Mengyao and Lin, Xixun and Qi, Xingqin and Yan, Guiying},
  journal={arXiv preprint arXiv:2604.10955},
  year={2026}
}

@book{hale2010asymptotic,
  title={Asymptotic behavior of dissipative systems},
  author={Hale, Jack K},
  number={25},
  year={2010},
  publisher={American Mathematical Soc.}
}

@book{temam2012infinite,
  title={Infinite-dimensional dynamical systems in mechanics and physics},
  author={Temam, Roger},
  year={2012},
  publisher={Springer Science \& Business Media}
}

@article{lindelof1894application,
  title={Sur l'application des m{\'e}thodes d'approximations successives {\`a} l'{\'e}tude des int{\'e}grales r{\'e}elles des {\'e}quations diff{\'e}rentielles ordinaires},
  author={Lindel{\"o}f, Ernest},
  journal={Journal de math{\'e}matiques pures et appliqu{\'e}es},
  volume={10},
  pages={117--128},
  year={1894}
}

@inproceedings{zhang2020hyper,
  title={Hyper-SAGNN: a self-attention based graph neural network for hypergraphs},
  author={Zhang, R and Zou, Y and Ma, J},
  booktitle={International Conference on Learning Representations (ICLR)},
  year={2020}
}

@inproceedings{feng2019hypergraph,
  title={Hypergraph neural networks},
  author={Feng, Yifan and You, Haoxuan and Zhang, Zizhao and Ji, Rongrong and Gao, Yue},
  booktitle={Proceedings of the AAAI conference on artificial intelligence},
  volume={33},
  number={01},
  pages={3558--3565},
  year={2019}
}

@article{yadati2019hypergcn,
  title={Hypergcn: A new method for training graph convolutional networks on hypergraphs},
  author={Yadati, Naganand and Nimishakavi, Madhav and Yadav, Prateek and Nitin, Vikram and Louis, Anand and Talukdar, Partha},
  journal={Advances in neural information processing systems},
  volume={32},
  year={2019}
}

@article{huang2021unignn,
  title={Unignn: a unified framework for graph and hypergraph neural networks},
  author={Huang, Jing and Yang, Jie},
  journal={arXiv preprint arXiv:2105.00956},
  year={2021}
}

@article{chien2021you,
  title={You are allset: A multiset function framework for hypergraph neural networks},
  author={Chien, Eli and Pan, Chao and Peng, Jianhao and Milenkovic, Olgica},
  journal={arXiv preprint arXiv:2106.13264},
  year={2021}
}

@inproceedings{heydari2022message,
  title={Message passing neural networks for hypergraphs},
  author={Heydari, Sajjad and Livi, Lorenzo},
  booktitle={International Conference on Artificial Neural Networks},
  pages={583--592},
  year={2022},
  organization={Springer}
}

@inproceedings{li2018deeper,
  title={Deeper insights into graph convolutional networks for semi-supervised learning},
  author={Li, Qimai and Han, Zhichao and Wu, Xiao-Ming},
  booktitle={Proceedings of the AAAI conference on artificial intelligence},
  volume={32},
  number={1},
  year={2018}
}

@article{oono2019graph,
  title={Graph neural networks exponentially lose expressive power for node classification},
  author={Oono, Kenta and Suzuki, Taiji},
  journal={arXiv preprint arXiv:1905.10947},
  year={2019}
}

@article{cai2020note,
  title={A note on over-smoothing for graph neural networks},
  author={Cai, Chen and Wang, Yusu},
  journal={arXiv preprint arXiv:2006.13318},
  year={2020}
}

@article{huang2020tackling,
  title={Tackling over-smoothing for general graph convolutional networks},
  author={Huang, Wenbing and Rong, Yu and Xu, Tingyang and Sun, Fuchun and Huang, Junzhou},
  journal={arXiv preprint arXiv:2008.09864},
  year={2020}
}

@article{zhao2022comprehensive,
  title={Comprehensive Analysis of Over-smoothing in Graph Neural Networks from Markov Chains Perspective},
  author={Zhao, Weichen and Wang, Chenguang and Han, Congying and Guo, Tiande},
  journal={arXiv preprint arXiv:2211.06605},
  year={2022}
}

@article{rong2019dropedge,
  title={Dropedge: Towards deep graph convolutional networks on node classification},
  author={Rong, Yu and Huang, Wenbing and Xu, Tingyang and Huang, Junzhou},
  journal={arXiv preprint arXiv:1907.10903},
  year={2019}
}

@article{zhao2019pairnorm,
  title={Pairnorm: Tackling oversmoothing in gnns},
  author={Zhao, Lingxiao and Akoglu, Leman},
  journal={arXiv preprint arXiv:1909.12223},
  year={2019}
}

@inproceedings{chen2020simple,
  title={Simple and deep graph convolutional networks},
  author={Chen, Ming and Wei, Zhewei and Huang, Zengfeng and Ding, Bolin and Li, Yaliang},
  booktitle={International conference on machine learning},
  pages={1725--1735},
  year={2020},
  organization={PMLR}
}

@inproceedings{liu2020towards,
  title={Towards deeper graph neural networks},
  author={Liu, Meng and Gao, Hongyang and Ji, Shuiwang},
  booktitle={Proceedings of the 26th ACM SIGKDD international conference on knowledge discovery \& data mining},
  pages={338--348},
  year={2020}
}

@inproceedings{xu2018representation,
  title={Representation learning on graphs with jumping knowledge networks},
  author={Xu, Keyulu and Li, Chengtao and Tian, Yonglong and Sonobe, Tomohiro and Kawarabayashi, Ken-ichi and Jegelka, Stefanie},
  booktitle={International conference on machine learning},
  pages={5453--5462},
  year={2018},
  organization={pmlr}
}

@inproceedings{choi2023gread,
  title={Gread: Graph neural reaction-diffusion networks},
  author={Choi, Jeongwhan and Hong, Seoyoung and Park, Noseong and Cho, Sung-Bae},
  booktitle={International conference on machine learning},
  pages={5722--5747},
  year={2023},
  organization={PMLR}
}

@article{eliasof2024graph,
  title={Graph neural reaction diffusion models},
  author={Eliasof, Moshe and Haber, Eldad and Treister, Eran},
  journal={SIAM Journal on Scientific Computing},
  volume={46},
  number={4},
  pages={C399--C420},
  year={2024},
  publisher={SIAM}
}

@article{chen2022preventing,
  title={Preventing over-smoothing for hypergraph neural networks},
  author={Chen, Guanzi and Zhang, Jiying and Xiao, Xi and Li, Yang},
  journal={arXiv preprint arXiv:2203.17159},
  year={2022}
}

@inproceedings{li2025deep,
  title={Deep hypergraph neural networks with tight framelets},
  author={Li, Ming and Fang, Yujie and Wang, Yi and Feng, Han and Gu, Yongchun and Bai, Lu and Lio, Pietro},
  booktitle={Proceedings of the AAAI Conference on Artificial Intelligence},
  volume={39},
  number={17},
  pages={18385--18392},
  year={2025}
}

@article{duta2023sheaf,
  title={Sheaf hypergraph networks},
  author={Duta, Iulia and Cassar{\`a}, Giulia and Silvestri, Fabrizio and Li{\`o}, Pietro},
  journal={Advances in Neural Information Processing Systems},
  volume={36},
  pages={12087--12099},
  year={2023}
}

@article{li2025implicit,
  title={Implicit hypergraph neural networks: A stable framework for higher-order relational learning with provable guarantees},
  author={Li, Xiaoyu and Tang, Guangyu and Jiang, Jiaojiao},
  journal={arXiv preprint arXiv:2508.09427},
  year={2025}
}

@article{lu2025hypergraph,
  title={Hypergraph neural diffusion networks},
  author={Lu, Fengcheng and Ng, Michael and Yip, Andy},
  journal={Neural Networks},
  pages={108271},
  year={2025},
  publisher={Elsevier}
}

@article{dong2020hnhn,
  title={Hnhn: Hypergraph networks with hyperedge neurons},
  author={Dong, Yihe and Sawin, Will and Bengio, Yoshua},
  journal={arXiv preprint arXiv:2006.12278},
  year={2020}
}

@inproceedings{wang2023hypergraph,
  title={From hypergraph energy functions to hypergraph neural networks},
  author={Wang, Yuxin and Gan, Quan and Qiu, Xipeng and Huang, Xuanjing and Wipf, David},
  booktitle={International Conference on Machine Learning},
  pages={35605--35623},
  year={2023},
  organization={PMLR}
}

@inproceedings{yan2024hypergraph,
  title={Hypergraph dynamic system},
  author={Yan, Jielong and Feng, Yifan and Ying, Shihui and Gao, Yue},
  booktitle={The twelfth international conference on learning representations},
  year={2024}
}

@article{choi2025hypergraph,
  title={Hypergraph Neural Sheaf Diffusion: A Symmetric Simplicial Set Framework for Higher-Order Learning},
  author={Choi, Seongjin and Kim, Gahee and Oh, Yong-Geun},
  journal={IEEE Access},
  year={2025},
  publisher={IEEE}
}

@inproceedings{zhao2025understanding,
  title={Understanding oversmoothing in diffusion-based GNNs from the perspective of operator semigroup theory},
  author={Zhao, Weichen and Wang, Chenguang and Wang, Xinyan and Han, Congying and Guo, Tiande and Yu, Tianshu},
  booktitle={Proceedings of the 31st ACM SIGKDD Conference on Knowledge Discovery and Data Mining V. 1},
  pages={2043--2054},
  year={2025}
}

@article{deidda2025rethinking,
  title={Rethinking oversmoothing in graph neural networks: A rank-based perspective},
  author={Deidda, Piero and Zhang, Kaicheng and Higham, Desmond and Tudisco, Francesco},
  journal={arXiv e-prints},
  pages={arXiv--2502},
  year={2025}
}

@inproceedings{thorpe2022grand++,
  title={GRAND++: Graph neural diffusion with a source term},
  author={Thorpe, Matthew and Nguyen, Tan Minh and Xia, Hedi and Strohmer, Thomas and Bertozzi, Andrea and Osher, Stanley and Wang, Bao},
  booktitle={International Conference on Learning Representations},
  year={2022}
}

@article{zhou2026tackling,
  title={Tackling Over-smoothing on Hypergraphs: A Ricci Flow-guided Neural Diffusion Approach},
  author={Zhou, Mengyao and Zhou, Zhiheng and Han, Xiao and Qi, Xingqin and Wang, Guanghui and Yan, Guiying},
  journal={arXiv preprint arXiv:2603.15696},
  year={2026}
}

@article{dua2017uci,
  title={UCI machine learning repository, 2017},
  author={Dua, Dheeru and Graff, Casey and others},
  journal={URL http://archive. ics. uci. edu/ml},
  volume={7},
  number={1},
  pages={62},
  year={2017}
}

@inproceedings{wu20153d,
  title={3d shapenets: A deep representation for volumetric shapes},
  author={Wu, Zhirong and Song, Shuran and Khosla, Aditya and Yu, Fisher and Zhang, Linguang and Tang, Xiaoou and Xiao, Jianxiong},
  booktitle={Proceedings of the IEEE conference on computer vision and pattern recognition},
  pages={1912--1920},
  year={2015}
}

@inproceedings{chen2003visual,
  title={On visual similarity based 3D model retrieval},
  author={Chen, Ding-Yun and Tian, Xiao-Pei and Shen, Yu-Te and Ouhyoung, Ming},
  booktitle={Computer graphics forum},
  pages={223--232},
  year={2003},
  organization={Wiley Online Library}
}

@article{chodrow2021generative,
  title={Generative hypergraph clustering: From blockmodels to modularity},
  author={Chodrow, Philip S and Veldt, Nate and Benson, Austin R},
  journal={Science Advances},
  volume={7},
  number={28},
  pages={eabh1303},
  year={2021},
  publisher={American Association for the Advancement of Science}
}

@article{fowler2006connecting,
  title={Connecting the congress: A study of cosponsorship networks},
  author={Fowler, James H},
  journal={Political analysis},
  volume={14},
  number={4},
  pages={456--487},
  year={2006},
  publisher={Cambridge University Press}
}

@article{fey2019fast,
  title={Fast graph representation learning with PyTorch Geometric},
  author={Fey, Matthias and Lenssen, Jan Eric},
  journal={arXiv preprint arXiv:1903.02428},
  year={2019}
}

@inproceedings{amburg2020clustering,
  title={Clustering in graphs and hypergraphs with categorical edge labels},
  author={Amburg, Ilya and Veldt, Nate and Benson, Austin},
  booktitle={Proceedings of the web conference 2020},
  pages={706--717},
  year={2020}
}

@inproceedings{xie2025k,
  title={K-hop hypergraph neural network: A comprehensive aggregation approach},
  author={Xie, Linhuang and Gao, Shihao and Liu, Jie and Yin, Ming and Jin, Taisong},
  booktitle={Proceedings of the AAAI Conference on Artificial Intelligence},
  volume={39},
  number={20},
  pages={21679--21687},
  year={2025}
}

@article{wang2022equivariant,
  title={Equivariant hypergraph diffusion neural operators},
  author={Wang, Peihao and Yang, Shenghao and Liu, Yunyu and Wang, Zhangyang and Li, Pan},
  journal={arXiv preprint arXiv:2207.06680},
  year={2022}
}

@article{bai2021hypergraph,
  title={Hypergraph convolution and hypergraph attention},
  author={Bai, Song and Zhang, Feihu and Torr, Philip HS},
  journal={Pattern Recognition},
  volume={110},
  pages={107637},
  year={2021},
  publisher={Elsevier}
}

@inproceedings{chamberlain2021grand,
  title={Grand: Graph neural diffusion},
  author={Chamberlain, Ben and Rowbottom, James and Gorinova, Maria I and Bronstein, Michael and Webb, Stefan and Rossi, Emanuele},
  booktitle={International conference on machine learning},
  pages={1407--1418},
  year={2021},
  organization={PMLR}
}

@article{deshpande2018contextual,
  title={Contextual stochastic block models},
  author={Deshpande, Yash and Sen, Subhabrata and Montanari, Andrea and Mossel, Elchanan},
  journal={Advances in Neural Information Processing Systems},
  volume={31},
  year={2018}
}

@article{ghoshdastidar2014consistency,
  title={Consistency of spectral partitioning of uniform hypergraphs under planted partition model},
  author={Ghoshdastidar, Debarghya and Dukkipati, Ambedkar},
  journal={Advances in Neural Information Processing Systems},
  volume={27},
  year={2014}
}

@inproceedings{chien2018community,
  title={Community detection in hypergraphs: Optimal statistical limit and efficient algorithms},
  author={Chien, I and Lin, Chung-Yi and Wang, I-Hsiang},
  booktitle={International conference on artificial intelligence and statistics},
  pages={871--879},
  year={2018},
  organization={PMLR}
}

\clearpage
\appendices
\section{Proof of Lemma~\ref{lem:normalized_diagonal_modulation}}
\label{app:Lemma1}
\begin{proof}[Proof of Lemma~\ref{lem:normalized_diagonal_modulation}]
First, by construction, \(\mathbf A_\theta(\mathbf X)\) is diagonal on the incidence space. Since each score \(a_\theta(\mathbf{X}_v(t), \mathbf{X}_e(t))\) is locally Lipschitz, and the softmax map is smooth on finite-dimensional spaces, the coefficients \(a_\theta(\mathbf{X}_v(t), \mathbf{X}_e(t))\) are locally Lipschitz on bounded subsets of \(\mathcal X\). Hence \(\mathbf A_\theta(\mathbf X)\) are locally Lipschitz with respect to \(\mathbf X\). This proves \textnormal{(i)}.

Next, for every \((e,v)\in\mathcal I\), we have
\begin{equation}
\begin{aligned}
1\ge a_\theta(\mathbf{X}_v(t), \mathbf{X}_e(t))&= (1-\varepsilon)\frac{\exp\big(r^{\theta}_{(e,v)}\big)}{\sum_{e' \ni v} \exp\big( r^{\theta}_{(e',v)} \big)}+\frac{\varepsilon}{d_v}\\
&\ge
\frac{\varepsilon}{|d_v|}
\ge
\frac{\varepsilon}{d_{\max}}
=
a_{\min}>0.
\end{aligned}
\end{equation}
Therefore, all diagonal entries of $\mathbf A_\theta(\mathbf X)$ are uniformly bounded above by $1$ and below by $a_{\min}$, which gives
\begin{equation}
\mathbf I_N\succeq \mathbf A_\theta(\mathbf X)\succeq a_{\min}\mathbf I_N.
\end{equation}
This proves \textnormal{(ii)}.

It remains to identify the null space of \(\mathbf G\). Since
\(
\mathbf G=\Omega_{\mathcal I}^{1/2}\mathbf P
\)
and \(\Omega_{\mathcal I}^{1/2}\) is a positive diagonal matrix~\cite{zhou2026hypergraph}, we have
\begin{equation}
\ker(\mathbf G)=\ker(\mathbf P).
\end{equation}
Let \(\mathbf f\in L(\mathcal V)\) satisfy \(\mathbf P\mathbf f=0\). By the definition of the hypergraph gradient, for every \((e,v)\in\mathcal I\),
\begin{equation}
\frac{\mathbf f(v)}{\sqrt{d_v}}
=
\frac{1}{|e|}
\sum_{u\in e}
\frac{\mathbf f(u)}{\sqrt{d_u}}.
\end{equation}
Thus the quantity \(\frac{\mathbf f(v)}{\sqrt{d_v}}
\) is equal to the average over every hyperedge containing \(v\). This implies that it is constant on each hyperedge. Since the hypergraph is connected, this constant propagates across the whole hypergraph. Hence there exists \(c\in\mathbb R\) such that
\begin{equation}
\frac{\mathbf f(v)}{\sqrt{d_v}}=c,
\qquad \forall v\in\mathcal V.
\end{equation}
Equivalently,
\(
\mathbf f=c,\;\mathbf D_v^{1/2}\mathbf 1=c\phi.
\)
Therefore,
\begin{equation}
\ker(\mathbf G)\subseteq \operatorname{span}\{\phi\}.
\end{equation}
The reverse inclusion follows directly from the gradient definition: if \(\mathbf f=c\phi\), then
\begin{equation}
\frac{\mathbf f(v)}{\sqrt{d_v}}=c
\end{equation}
for every \(v\in\mathcal V\), and hence \(\mathbf P\mathbf f=0\). Thus
\begin{equation}
\ker(\mathbf G)=\operatorname{span}\{\phi\}.
\end{equation}
For \(\mathbf X\in\mathbb R^{n\times d}\), the same argument applies columnwise, yielding
\begin{equation}
\ker(\mathbf G)=\mathcal S_\phi
=
\{\phi c:\ c\in\mathbb R^{1\times d}\}.
\end{equation}
This proves \textnormal{(iii)}.
\end{proof}

\section{Proof of Theorem~\ref{thm:null_mode_attraction}}
\label{app:Theorem3}
\begin{proof}[Proof of Theorem~\ref{thm:null_mode_attraction}]
We first prove the invariance of \(\mathcal S_\phi\). By Lemma~\ref{lem:normalized_diagonal_modulation}, we have
\(\ker(\mathbf G)=\mathcal S_\phi.\) For any \(\mathbf X_0\in\mathcal S_\phi\), it follows that \(\mathbf G\mathbf X_0=0.\)
Substituting this into the diffusion equation gives
\begin{equation}
\frac{\partial \mathbf X(t)}{\partial t}
=-\mathbf G^\top
\mathbf A_\theta(\mathbf X(t))
\mathbf G\mathbf X(t)=0
\end{equation}
whenever \(\mathbf X(t)\in\mathcal S_\phi\). Therefore, the solution starting from \(\mathbf X_0\in\mathcal S_\phi\) remains fixed:
\begin{equation}
\mathbf S_D(t)\mathbf X_0=\mathbf X_0,\quad t\ge0.
\end{equation}
Hence \(\mathcal S_\phi\) is invariant under the diffusion semiflow.

We next prove the exponential contraction of the transverse component. Let \(\mathbf Z(t):=\mathbf Q_\phi \mathbf X(t).
\) Since \(\mathcal S_\phi=\ker(\mathbf G)\), we have
\begin{equation}
\mathbf G\mathbf X(t)
=\mathbf G\mathbf Q_\phi\mathbf X(t)=
\mathbf G\mathbf Z(t).
\end{equation}
Moreover, the diffusion operator
\begin{equation}
\mathbf L_\theta(\mathbf X)
:=
\mathbf G^\top\mathbf A_\theta(\mathbf X)\mathbf G
\end{equation}
is symmetric positive semidefinite because \(\mathbf A_\theta(\mathbf X)\) is diagonal positive semidefinite. Since \(\mathbf L_\theta(\mathbf X)\) annihilates \(\mathcal S_\phi\), the transverse dynamics satisfies
\begin{equation}
\frac{d}{dt}\mathbf Z(t)
= \mathbf Q_\phi\frac{d}{dt}\mathbf X(t)
=-\mathbf L_\theta(\mathbf X(t))\mathbf Z(t).
\end{equation}

Define the transverse energy \(V(t):=\frac12\|\mathbf Z(t)\|_F^2.\)
Then
\begin{equation}
\begin{aligned}
\dot V(t)
&=
\left\langle
\mathbf Z(t),
\frac{d}{dt}\mathbf Z(t)
\right\rangle_F 
=
-\left\langle
\mathbf Z(t),
\mathbf G^\top\mathbf A_\theta(\mathbf X(t))\mathbf G\mathbf Z(t)
\right\rangle_F  \\
&=
-\left\langle
\mathbf G\mathbf Z(t),
\mathbf A_\theta(\mathbf X(t))\mathbf G\mathbf Z(t)
\right\rangle_F .
\end{aligned}
\end{equation}
By Lemma~\ref{lem:normalized_diagonal_modulation},
\(\mathbf A_\theta(\mathbf X(t))\succeq a_{\min}\mathbf I_N.
\) Hence
\begin{equation}
\dot V(t)
\le
-a_{\min}\|\mathbf G\mathbf Z(t)\|_F^2.
\end{equation}

It remains to use the spectral gap on \(\mathcal S_\phi^\perp\). Since \(\ker(\mathbf G)=\mathcal S_\phi,
\) the quantity 
\begin{equation}
\lambda_2:=\min_{\mathbf Z\in\mathcal S_\phi^\perp,\ \|\mathbf Z\|_F=1}\|\mathbf G\mathbf Z\|_F^2
\end{equation} 
is strictly positive. Indeed, if \(\lambda_2=0\), then there exists a nonzero \(\mathbf Z\in\mathcal S_\phi^\perp\) such that \(\mathbf G\mathbf Z=0\), which implies \(\mathbf Z\in\ker(\mathbf G)=\mathcal S_\phi\), contradicting \(\mathbf Z\in\mathcal S_\phi^\perp\) and \(\|\mathbf Z\|_F=1\). Therefore, \(\|\mathbf G\mathbf Z(t)\|_F^2\ge\lambda_2\|\mathbf Z(t)\|_F^2.\) Consequently,
\begin{equation}
\dot V(t)\le-a_{\min}\lambda_2|\mathbf Z(t)|_F^2
=-2\gamma_D V(t),
\end{equation} 
where \(\gamma_D:=a_{\min}\lambda_2>0.\)
By Gronwall's inequality,
\begin{equation}
V(t)\le e^{-2\gamma_D t}V(0).
\end{equation} 
Taking square roots yields
\begin{equation}
\begin{aligned}
\|\mathbf Q_\phi S_D(t)\mathbf X_0\|_F
&=\|\mathbf Z(t)\|_F\le e^{-\gamma_D t}\|\mathbf Z(0)\|_F\\
&=e^{-\gamma_D t}\|\mathbf Q_\phi\mathbf X_0\|_F.
\end{aligned}
\end{equation} 
Since \(\mathbf \Pi_\phi\) is the orthogonal projector onto \(\mathcal S_\phi\), we have
\begin{equation}
\operatorname{dist}(\mathbf S_D(t)\mathbf X_0,\mathcal S_\phi)
=\|\mathbf Q_\phi \mathbf S_D(t)\mathbf X_0\|_F.
\end{equation}
\end{proof}

\section{Proof of Theorem~\ref{thm:diffusion_dirichlet_decay}}
\label{app:Theorem4}
\begin{proof}[Proof of Theorem~\ref{thm:diffusion_dirichlet_decay}]
Let \(\mathbf Z(t):=\mathbf Q_\phi \mathbf S_D(t)\mathbf X_0.\) By definition,
\begin{equation}
\mathcal E_\phi(\mathbf S_D(t)\mathbf X_0)
=
\frac12
\left\langle
\mathbf Z(t),
\mathbf L_H\mathbf Z(t)
\right\rangle_F .
\end{equation}
Since \(\mathbf L_H\) is symmetric positive semidefinite,
\begin{equation}
\left\langle
\mathbf Z(t),
\mathbf L_H\mathbf Z(t)
\right\rangle_F
\le
\lambda_{\max}(\mathbf L_H)
\|\mathbf Z(t)\|_F^2.
\end{equation}
Therefore,
\begin{equation}
\mathcal E_\phi(\mathbf S_D(t)\mathbf X_0)
\le
\frac12
\lambda_{\max}(\mathbf L_H)
\|\mathbf Q_\phi \mathbf S_D(t)\mathbf X_0\|_F^2.
\end{equation}
By the transverse contraction estimate in Theorem~\ref{thm:null_mode_attraction}, we have \(
\|\mathbf Q_\phi \mathbf S_D(t)\mathbf X_0\|_F^2 \le e^{-2\gamma_D t}
\|\mathbf Q_\phi\mathbf X_0\|_F^2.\) Thus,
\begin{equation}
\mathcal E_\phi(\mathbf S_D(t)\mathbf X_0)
\le
\frac12
\lambda_{\max}(\mathbf L_H)
e^{-2\gamma_D t}
\|\mathbf Q_\phi\mathbf X_0\|_F^2,
\quad t\ge0.
\end{equation}
In particular, when \(t\to\infty\), we have
\(
\mathcal E_\phi(\mathbf S_D(t)\mathbf X_0)\to0.
\)

It remains to prove the equivalence between Dirichlet-energy decay and attraction to \(\mathcal S_\phi\). Since
\(\mathbf L_H\phi=0,\) and \(\ker(\mathbf L_H)=\operatorname{span}{\phi},\) the restriction of \(\mathbf L_H\) to \(\mathcal S_\phi^\perp\) is positive definite. Let
\begin{equation}
\lambda_{H,2}
:=
\min_{\mathbf Z\in\mathcal S_\phi^\perp,\;\|\mathbf Z\|_F=1}
\langle \mathbf Z,\mathbf L_H\mathbf Z\rangle_F>0
\end{equation}
denote its smallest positive eigenvalue on \(\mathcal S_\phi^\perp\). Then for every \(\mathbf X\in\mathcal X\),
\begin{equation}
\frac12\lambda_{H,2}\|\mathbf Q_\phi\mathbf X\|_F^2
\le\mathcal E_\phi(\mathbf X)
\le\frac12\lambda_{\max}(\mathbf L_H)
\|\mathbf Q_\phi\mathbf X\|_F^2.
\end{equation}
The upper bound implies that
\begin{equation}
\operatorname{dist}(\mathbf S_D(t)\mathbf X_0,\mathcal S_\phi)
=\|\mathbf Q_\phi \mathbf S_D(t)\mathbf X_0\|_F\to0
\end{equation}
implies \(\mathcal E_\phi(\mathbf S_D(t)\mathbf X_0)\to0.\)
Conversely, the lower bound implies that \(\mathcal E_\phi(\mathbf S_D(t)\mathbf X_0)\to0\), then \(\|\mathbf Q_\phi \mathbf S_D(t)\mathbf X_0\|_F\to0.\) 

Since \(\operatorname{dist}(\mathbf S_D(t)\mathbf X_0,\mathcal S_\phi)=\|\mathbf Q_\phi \mathbf S_D(t)\mathbf X_0\|_F,\)
we obtain
\begin{equation}
\mathcal E_\phi(\mathbf S_D(t)\mathbf X_0)\to0
\quad\Longleftrightarrow\quad
\operatorname{dist}(\mathbf S_D(t)\mathbf X_0,\mathcal S_\phi)\to0.
\end{equation}
\end{proof}

\section{Proof of Lemma~\ref{lem:bounded_reaction_coefficient}}
\label{app:Lemma5}
\begin{proof}[Proof of Lemma~\ref{lem:bounded_reaction_coefficient}]
Let \(\mathbf Z:=\mathbf Q_\phi\mathbf X .\) By definition,
\begin{equation}
R_\theta(\mathbf X)
=
\frac{
\left\langle
\mathbf G\mathbf Z,
\mathbf A_\theta(\mathbf X)\mathbf G\mathbf Z
\right\rangle_F
}{
\|\mathbf Z\|_F^2
}.
\end{equation}
Since \(\mathbf A_\theta(\mathbf X)\succeq0\), we have
\(
\left\langle
\mathbf G\mathbf Z,
\mathbf A_\theta(\mathbf X)\mathbf G\mathbf Z
\right\rangle_F
\ge0.
\)
Therefore, \(R_\theta(\mathbf X)\ge0.\)

Next, because the normalized diffusion modulation \(\mathbf A_\theta(\mathbf X)\preceq \mathbf I_N,\) we obtain \(
\left\langle
\mathbf G\mathbf Z,
\mathbf A_\theta(\mathbf X)\mathbf G\mathbf Z
\right\rangle_F
\le
\|\mathbf G\mathbf Z\|_F^2.
\) Since \(\mathbf L_H=\mathbf G^\top\mathbf G\), it follows that
\(
\|\mathbf G\mathbf Z\|_F^2
=
\left\langle
\mathbf Z,
\mathbf L_H\mathbf Z
\right\rangle_F .
\)
By the Rayleigh quotient bound,
\begin{equation}
\left\langle
\mathbf Z,
\mathbf L_H\mathbf Z
\right\rangle_F
\le
\lambda_{\max}(\mathbf L_H)\|\mathbf Z\|_F^2.
\end{equation}
For matrix-valued node representations, this identity is understood columnwise. Writing
\(\mathbf Z=[\mathbf z_1,\ldots,\mathbf z_d]\), we have
\begin{equation}
\left\langle
\mathbf Z,
\mathbf L_H\mathbf Z
\right\rangle_F
=
\sum_{j=1}^d
\mathbf z_j^\top \mathbf L_H \mathbf z_j .
\end{equation}
By the Rayleigh quotient bound for the symmetric positive semidefinite matrix \(\mathbf L_H\),
\begin{equation}
\mathbf z_j^\top \mathbf L_H \mathbf z_j
\le
\lambda_{\max}(\mathbf L_H)\|\mathbf z_j\|_2^2,
\quad j=1,\ldots,d .
\end{equation}
Summing over all feature channels yields
\begin{equation}
\left\langle
\mathbf Z,
\mathbf L_H\mathbf Z
\right\rangle_F
\le
\lambda_{\max}(\mathbf L_H)\|\mathbf Z\|_F^2 .
\end{equation}
Thus, we have
\begin{equation}
R_\theta(\mathbf X)
\le
\frac{
\lambda_{\max}(\mathbf L_H)\|\mathbf Z\|_F^2
}{
\|\mathbf Z\|_F^2
}
\le
\lambda_{\max}(\mathbf L_H).
\end{equation}
which implies that
\(
0\le R_\theta(\mathbf X)\le \lambda_{\max}(\mathbf L_H).
\)

Finally, since the hyperbolic tangent is uniformly bounded,
\(
-1\le
\tanh
\left(
\tau_\eta-\|\mathbf Q_\phi\mathbf X\|_F^2
\right)
\le1,
\)
we have
\[
-1
\le
R_\theta(\mathbf X)
+
\tanh
\left(
\tau_\eta-\|\mathbf Q_\phi\mathbf X\|_F^2
\right)
\le
\lambda_{\max}(\mathbf L_H)+1.
\]
Therefore, the scalar reaction coefficient is uniformly bounded. 
\end{proof}

\section{Proof of Theorem~\ref{thm:hnrd_wellposedness}}
\label{app:Theorem6}
\begin{proof}[Proof of Theorem~\ref{thm:hnrd_wellposedness}]
Define the HNRD vector field by
\begin{equation}
\mathcal H_{\theta,\eta}(\mathbf X)
:=-\mathbf G^\top
\mathbf A_\theta(\mathbf X)
\mathbf G\mathbf X
+\mathcal R_\eta(\mathbf X).
\end{equation}
Then Eq.~\eqref{eq:hnrd_equation} can be written as
\begin{equation}
\frac{d\mathbf X(t)}{dt}
=\mathcal H_{\theta,\eta}(\mathbf X(t)).
\end{equation}

We first show local well-posedness. By Lemma~\ref{lem:normalized_diagonal_modulation}, the diffusion modulation \(\mathbf A_\theta(\mathbf X)\) is locally Lipschitz with respect to \(\mathbf X\). Hence the diffusion field \(\mathbf X\mapsto
-\mathbf G^\top\mathbf A_\theta(\mathbf X)
\mathbf G\mathbf X\) is locally Lipschitz on the finite-dimensional phase space \(\mathcal X\). Moreover, the map \(R_\theta(\mathbf X)\) is locally Lipschitz. Since \(\tanh(\cdot)\), \(\|\mathbf Q_\phi\mathbf X\|_F^2\), and \(\mathbf Q_\phi\mathbf X\) are also locally Lipschitz, the reaction field \(\mathcal R_\eta(\mathbf X)\) is locally Lipschitz. Therefore, the full vector field \(\mathcal H_{\theta,\eta}\) is locally Lipschitz on \(\mathcal X\).

By the Picard--Lindelöf theorem~\cite{lindelof1894application}, for every initial condition \(\mathbf X_0\in\mathcal X\), there exists a unique maximal classical solution
\begin{equation}
\mathbf X(t;\mathbf X_0)\in C^1([0,T_{\max});\mathcal X).
\end{equation}

It remains to show that \(T_{\max}=+\infty\). By the normalized construction of \(\mathbf A_\theta\), we have
\(
0\preceq \mathbf A_\theta(\mathbf X)\preceq \mathbf I_N .
\)
Thus,
\begin{equation}
\|
\mathbf G^\top
\mathbf A_\theta(\mathbf X)
\mathbf G\mathbf X
\|_F\le\|\mathbf G\|_2^2\|\mathbf X\|_F .
\end{equation}
On the other hand, Lemma~\ref{lem:bounded_reaction_coefficient} gives
\begin{equation}
\|
R_\theta(\mathbf X)
+
\tanh
\left(
\tau_\eta-\|\mathbf Q_\phi\mathbf X\|_F^2
\right)
\|
\le
\lambda_{\max}(\mathbf L_H)+1 .
\end{equation}
Therefore, we have
\begin{equation}
\|\mathcal R_\eta(\mathbf X)\|_F
\le
\left(
\lambda_{\max}(\mathbf L_H)+1
\right)
\|\mathbf X\|_F .
\end{equation}
Combining the above estimates yields the linear growth bound
\begin{equation}
\|\mathcal H_{\theta,\eta}(\mathbf X)\|_F
\le C\|\mathbf X\|_F,
\end{equation}
where
\(
C:=\|\mathbf G\|_2^2+\lambda_{\max}(\mathbf L_H)+1 .
\)
Along the maximal solution, we have
\begin{equation}
\frac{d}{dt}\|\mathbf X(t)\|_F
\le\|\dot{\mathbf X}(t)\|_F
=\|\mathcal H_{\theta,\eta}(\mathbf X(t))\|_F
\le C\|\mathbf X(t)\|_F .
\end{equation}
By Gronwall's inequality,
\begin{equation}
\|\mathbf X(t)\|_F
\le e^{Ct}\|\mathbf X_0\|_F,
\quad
t\in[0,T_{\max}) .
\end{equation}
Hence the solution cannot blow up in finite time. By the standard blow-up alternative for finite-dimensional ODEs, we must have \(T_{\max}=+\infty .\)

Therefore, Eq.~\eqref{eq:hnrd_equation} admits a unique global classical solution for every \(\mathbf X_0\in\mathcal X\). Finally, uniqueness and continuous dependence on initial data imply that the solution operator
\begin{equation}
S_R(t)\mathbf X_0:=\mathbf X(t;\mathbf X_0),
\quad t\ge0,
\end{equation}
defines a global continuous semiflow \(\{S_R(t)\}_{t\ge0}\) on \(\mathcal X\).     
\end{proof}

\section{Proof of Theorem~\ref{thm:hnrd_transverse_stabilization}}
\label{app:theorem7}
\begin{proof}[Proof of Theorem~\ref{thm:hnrd_transverse_stabilization}]
Let \(\mathbf X(t):=\mathbf S_R(t)\mathbf X_0\). Since \(\mathbf Q_\phi\mathbf X_0\neq0\), we have \(s(0)>0\). The HNRD equation can be written as
\begin{equation}
\begin{aligned}
\frac{d\mathbf X(t)}{dt}
&=
-\mathbf G^\top
\mathbf A_\theta(\mathbf X(t))
\mathbf G\mathbf X(t)\\
&+
\left[
R_\theta(\mathbf X(t))
+
\tanh(\tau_\eta-s(t))
\right]
\mathbf Z(t).
\end{aligned}
\end{equation}
Because \(\ker(\mathbf G)=\mathcal S_\phi\), we have
\(\mathbf G\mathbf X(t)=\mathbf G\mathbf Q_\phi\mathbf X(t)
=\mathbf G\mathbf Z(t).\) Moreover, the diffusion term belongs to the transverse subspace, and the reaction term is also proportional to \(\mathbf Z(t)\). Hence,
\begin{equation}
\frac{d}{dt}s(t)
=
2\left\langle
\mathbf Z(t),
\frac{d}{dt}\mathbf Z(t)
\right\rangle_F .
\end{equation}
Substituting the HNRD dynamics gives
\begin{equation}
\begin{aligned}
\frac{d}{dt}s(t)
&=-2\left\langle
\mathbf G\mathbf Z(t),
\mathbf A_\theta(\mathbf X(t))
\mathbf G\mathbf Z(t)
\right\rangle_F  \\
&+2\left[R_\theta(\mathbf X(t))
+\tanh(\tau_\eta-s(t))\right]
\|\mathbf Z(t)\|_F^2 .
\end{aligned}
\end{equation}
By the definition of \(R_\theta\),
\begin{equation}
R_\theta(\mathbf X(t))\|\mathbf Z(t)\|_F^2
=\left\langle
\mathbf G\mathbf Z(t),
\mathbf A_\theta(\mathbf X(t))
\mathbf G\mathbf Z(t)
\right\rangle_F .
\end{equation}
Therefore, the diffusion dissipation is exactly compensated, and we obtain
\begin{equation}
\frac{d}{dt}s(t)=2\tanh(\tau_\eta-s(t))s(t).
\end{equation}

We now analyze this scalar equation. Since \(s(0)>0\), the solution remains positive for all \(t\ge0\). If \(0<s(t)<\tau_\eta\), then \(\tanh(\tau_\eta-s(t))>0,\) so \(s(t)\) increases. If \(s(t)>\tau_\eta\), then
\(\tanh(\tau_\eta-s(t))<0,\) so \(s(t)\) decreases. Thus \(s(t)\) is driven toward the unique positive equilibrium \(s=\tau_\eta\). Consequently,
\begin{equation}
\lim_{t\to\infty}s(t)=\|\mathbf Q_\phi \mathbf S_R(t)\mathbf X_0\|_F^2=\tau_\eta.
\end{equation}

Finally, since \(\mathbf Q_\phi\) is the orthogonal projector onto \(\mathcal S_\phi^\perp\), we have
\(\operatorname{dist}\left(\mathbf S_R(t)\mathbf X_0,
\mathcal S_\phi\right)=\|\mathbf Q_\phi \mathbf S_R(t)\mathbf X_0\|_F .\)
Therefore,
\begin{equation}
\lim_{t\to\infty}
\operatorname{dist}^2
\left(
\mathbf S_R(t)\mathbf X_0,
\mathcal S_\phi
\right)=\tau_\eta .
\end{equation} 
\end{proof}

\section{Proof of Theorem~\ref{thm:hnrd_dirichlet_noncollapse}}
\label{app:theorem8}
\begin{proof}[Proof of Theorem~\ref{thm:hnrd_dirichlet_noncollapse}]
Since \(Q_\phi\) is the orthogonal projector onto \(\mathcal S_\phi^\perp\), we have \(\mathbf Z(t)\in\mathcal S_\phi^\perp .\) 

Because \(\mathbf L_H\) is symmetric positive semidefinite and satisfies \(\ker(\mathbf L_H)=\operatorname{span}\{\phi\},\) its restriction to \(\mathcal S_\phi^\perp\) is positive definite. Therefore, for every
\(\mathbf Z\in\mathcal S_\phi^\perp\),
\begin{equation}
\lambda_{H,2}\|\mathbf Z\|_F^2
\le
\left\langle
\mathbf Z,
\mathbf L_H\mathbf Z
\right\rangle_F
\le
\lambda_{H,\max}\|\mathbf Z\|_F^2 .
\end{equation}
Applying this estimate to \(\mathbf Z(t)=\mathbf Q_\phi \mathbf S_R(t)\mathbf X_0\), we obtain
\begin{equation}
\begin{aligned}
\frac{\lambda_{H,2}}{2}
\|\mathbf Q_\phi \mathbf S_R(t)\mathbf X_0\|_F^2
&\le
\mathcal E_\phi(\mathbf S_R(t)\mathbf X_0)\\
&\le
\frac{\lambda_{H,\max}}{2}
\|\mathbf Q_\phi \mathbf S_R(t)\mathbf X_0\|_F^2 .
\end{aligned}
\end{equation}

By Theorem~\ref{thm:hnrd_transverse_stabilization}, when \(t\to\infty\), we have \(\|\mathbf Q_\phi \mathbf S_R(t)\mathbf X_0\|_F^2\to\tau_\eta.\)
Taking the lower limit and upper limit in the above two-sided estimate gives
\begin{equation}
\liminf_{t\to\infty}
\mathcal E_\phi(\mathbf S_R(t)\mathbf X_0)
\ge
\frac{\lambda_{H,2}}{2}\tau_\eta
\end{equation}
and
\begin{equation}
\limsup_{t\to\infty}
\mathcal E_\phi(\mathbf S_R(t)\mathbf X_0)
\le
\frac{\lambda_{H,\max}}{2}\tau_\eta .
\end{equation}
Combining these two inequalities yields
\begin{equation}
\begin{aligned}
\frac{\lambda_{H,\max}}{2}\tau_\eta
&\ge
\limsup_{t\to\infty}
\mathcal E_\phi(\mathbf S_R(t)\mathbf X_0)\\
&\ge
\liminf_{t\to\infty}
\mathcal E_\phi(\mathbf S_R(t)\mathbf X_0)
\ge
\frac{\lambda_{H,2}}{2}\tau_\eta .
\end{aligned}
\end{equation}
\end{proof}

\section{Proof of Theorem~\ref{thm:discrete_hnrd_stepsize_stability}}
\label{app:theorem9}
\begin{proof}[Proof of Theorem~\ref{thm:discrete_hnrd_stepsize_stability}]
Let \(\mathbf Z_k:=\mathbf Q_\phi\mathbf X^k,\)
\(s_k:=\|\mathbf Z_k\|_F^2,\) and
\(
\mathbf B_k:=
\mathbf G^\top
\mathbf A_\theta(\mathbf X^k)
\mathbf G .
\)
For simplicity, we write \(T_k:=\tanh(\tau_\eta-s_k).
\) The discrete HNRD layer gives the transverse update
\begin{equation}
\mathbf Z_{k+1}
=
\mathbf Z_k
+
h\left[
-\mathbf B_k\mathbf Z_k
+
\left(
R_\theta(\mathbf X^k)+T_k
\right)
\mathbf Z_k
\right].
\end{equation}
Equivalently,
\begin{equation}
\mathbf Z_{k+1}
=(1+hT_k)\mathbf Z_k-
h
\left(
\mathbf B_k
-
R_\theta(\mathbf X^k)\mathbf I
\right)
\mathbf Z_k .
\end{equation}

By definition,
\(
R_\theta(\mathbf X^k)
=
\frac{
\left\langle
\mathbf Z_k,
\mathbf B_k\mathbf Z_k
\right\rangle_F
}{
\|\mathbf Z_k\|_F^2
}.
\)
Hence, we have
\begin{equation}
\left\langle\mathbf Z_k,\left(\mathbf B_k-R_\theta(\mathbf X^k)\mathbf I\right)\mathbf Z_k\right\rangle_F=0.
\end{equation}
Therefore, the cross term vanishes when computing \(\|\mathbf Z_{k+1}\|_F^2\), and we obtain
\begin{equation}
s_{k+1}=(1+hT_k)^2s_k+h^2\|\left(\mathbf B_k-R_\theta(\mathbf X^k)\mathbf I\right)
\mathbf Z_k\|_F^2.
\end{equation}

By the normalized hypergraph diffusion assumption,
\begin{equation}
0\preceq \mathbf B_k\preceq \lambda_{\max}(\mathbf L_H)\mathbf I .
\end{equation}
Thus, using the standard variance bound for a symmetric operator whose spectrum lies in
\([0,\lambda_{\max}(\mathbf L_H)]\), we have
\begin{equation}
\|\left(
\mathbf B_k
-R_\theta(\mathbf X^k)\mathbf I
\right)
\mathbf Z_k
\|_F^2
\le
\frac{\lambda_{\max}^2(\mathbf L_H)}{4}s_k .
\end{equation}
Consequently,
\begin{equation}
s_{k+1}\le\left[(1+hT_k)^2+\frac{h^2}{4}\lambda_{\max}^2(\mathbf L_H)\right]s_k.
\end{equation}

Since the hypergraph Laplacian satisfies
\begin{equation}
0\le \lambda_{\max}(\mathbf L_H)\le2
\end{equation}
and \(0<h<1\), we can choose a constant \(\delta\in(0,1)\) such that
\begin{equation}
2h\delta 
>\left(
h^2\delta^2+
\frac{h^2}4\lambda_{\max}^2(\mathbf L_H)
\right).
\end{equation}
Because as \(s\to+\infty\), we have \(\tanh(\tau_\eta-s)\to -1,\) there exists \(M>0\) such that, whenever \(s_k\ge M\),
\begin{equation}
T_k=\tanh(\tau_\eta-s_k)\le -\delta .
\end{equation}
For such \(s_k\), we obtain
\begin{equation}
s_{k+1}
\le
\left[
(1-h\delta)^2
+
\frac{h^2}{4}\lambda_{\max}^2(\mathbf L_H)
\right]s_k .
\end{equation}
Set \(\rho:=(1-h\delta)^2+\frac{h^2}{4}\lambda_{\max}^2(\mathbf L_H).\) By the choice of \(\delta\), we have \(\rho<1\). Therefore, whenever \(s_k\ge M\),
\begin{equation}
s_{k+1}\le \rho s_k<s_k .
\end{equation}

On the other hand, if \(s_k<M\), then \(T_k\le1\), and hence
\begin{equation}
s_{k+1}\le\left[(1+h)^2+\frac{h^2}{4}\lambda_{\max}^2(\mathbf L_H)\right]M .
\end{equation}
Define \(C_h:=\max\{s_0,M,\left[(1+h)^2+\frac{h^2}{4}\lambda_{\max}^2(\mathbf L_H)\right]M\}.\)
Then the above two estimates imply by induction that
\begin{equation}
s_k\le C_h,
\quad
\forall k\ge0 .
\end{equation}
Thus the transverse energy remains uniformly bounded along depth.     
\end{proof}

\end{document}